\documentclass[journal]{IEEEtran}

\usepackage{moreverb,url}
\usepackage[colorlinks,bookmarksopen,bookmarksnumbered,citecolor=red,urlcolor=red]{hyperref}
\usepackage{graphics} 
\usepackage{epsfig} 
\usepackage{amsmath} 
\usepackage{amsthm} 
\usepackage{amssymb}  
\usepackage{fourier} 
\usepackage{caption}
\usepackage{subcaption}
\usepackage{cite}
\usepackage[lined,linesnumbered,ruled,vlined]{algorithm2e} 
\usepackage{times}
\usepackage{multirow} 
\usepackage{color}

\newtheorem{theorem}{Theorem}
\newtheorem{cor}{Corollary}
\newtheorem{lemma}{Lemma}

\newcommand{\vect}[1]{\boldsymbol{#1}}
\newcommand{\blue}[1]{\textcolor{blue}{#1}}

\usepackage{fancyhdr} 
\fancypagestyle{firstpage}{
	\fancyhf{}
	\fancyhead[C]{\textcolor{red}{The paper will appear in the IEEE Transactions on Robotics © 2022 IEEE.  Personal use of this material is permitted} }
}

\begin{document}
%
\title{\blue{Probabilistic network topology prediction for active planning:} An adaptive algorithm and application}
%
%
%

\author{Liang~Zhang, 
        Zexu~Zhang, 
        Roland~Siegwart,~\IEEEmembership{Fellow,~IEEE,}
        Jen~Jen~Chung,~\IEEEmembership{Member,~IEEE}
\thanks{Corresponding author: Zexu Zhang}
\thanks{Liang Zhang is with the School of Engineering and Automation, Anhui University, Hefei 230601, China and with with the Deep Space Exploration and Research Center, School of Astronautics, Harbin Institute of Technology, Harbin 150001, China. Liang Zhang is also a visiting student in Autonomous System Lab, ETH Z{\" u}rich 8092, Switzerland when developing this work  {\tt\small liangzhang@ahu.edu.cn}. }  

\thanks{Zexu Zhang is with the Deep Space Exploration and Research Center, School of Astronautics, Harbin Institute of Technology, Harbin 150001, China.  Zexu Zhang is also with the Shaanxi Key Laboratory of Integrated and Intelligent Navigation. {\tt\small zexuzhang@hit.edu.cn}}
\thanks{ Jen Jen Chung and Roland Siegwart are with the Autonomous Systems Lab, ETH Z{\" u}rich, Z{\"u}rich 8092, Switzerland. {\tt\small\{rsiegwart;chungj\}@ethz.ch}.}
\thanks{Research supported in part by the Basic Research Strengthening Program of China(173 Program) (2020-JCJQ-ZD-015-00), the National Natural Science Foundation of China under Granted 61374213 and 61573247, and the Open Foundation
from Shanxi Key Laboratory of Integrated and Intelligent Navigation (SKLIIN-20180208). Part of this work is also supported by the scholarship from the China Scholarship Council.}
}

\markboth{Journal of \LaTeX\ Class Files,~Vol.~14, No.~8, August~2015}%
{Shell \MakeLowercase{\textit{et al.}}: Bare Demo of IEEEtran.cls for IEEE Journals}


\maketitle

\thispagestyle{firstpage}

\begin{abstract}
This paper tackles the problem of active planning to achieve cooperative localization for multi-robot systems (MRS) under measurement uncertainty in GNSS-limited scenarios. Specifically, we address the issue of accurately predicting the probability of a future connection between two robots equipped with range-based measurement devices. Due to the limited range of the equipped sensors, edges in the network connection topology will be created or destroyed as the robots move with respect to one another. Accurately predicting the future existence of an edge, given imperfect state estimation and noisy actuation, is therefore a challenging task.
An adaptive power series expansion (or APSE) algorithm is developed based on current estimates and control candidates. Such an algorithm applies the power series expansion formula of the quadratic positive form in a normal distribution. Finite-term approximation is made to realize the computational tractability. Further analyses are presented to show that the truncation error in the finite-term approximation can be theoretically reduced to a desired threshold by adaptively choosing the summation degree of the power series. Several sufficient conditions are rigorously derived as the selection principles. Finally, extensive simulation results and comparisons, with respect to both single and multi-robot cases, validate that a formally computed and therefore more accurate probability of future topology can help improve the performance of active planning under uncertainty.
\end{abstract}


%
\IEEEpeerreviewmaketitle

\section{Introduction}
%
%
%
%
\IEEEPARstart{A}{ctive} planning under uncertainty plays an important role for robots to autonomously navigate and operate in noisy environments by optimally choosing either systematic configurations or control inputs to minimize (or maximize) an aggregated objective function comprising various requirements from both task and safety needs. It has recently attracted more and more attention in a variety of applications in SLAM \cite{cadena2016past}, sensor deployments \cite{hidaka2005optimal}, surveillance \cite{schlotfeldt2018anytime}, search \& rescue \cite{patwari2005locating} and so on. 

\begin{figure*}[htbp] 
	\centering
	\includegraphics[width=0.8\linewidth]{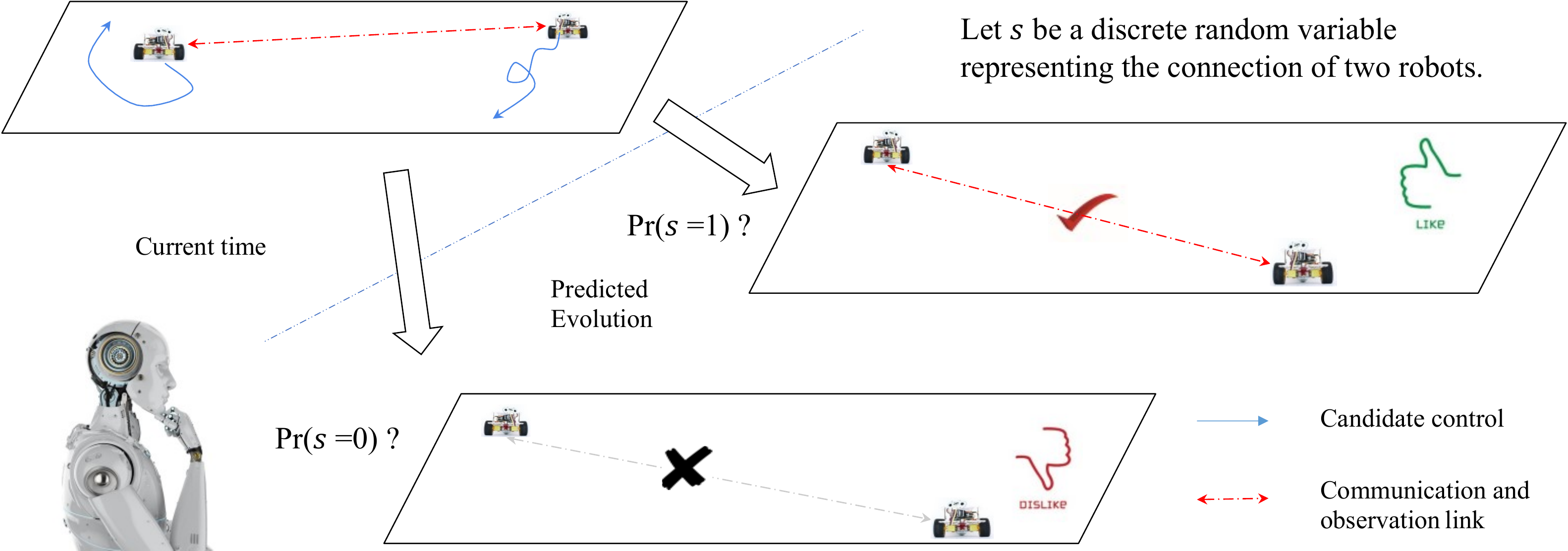}
	\caption{Overall problem description of nondeterministic future network. A connection  indicates a communication and (or) measurement event between the two associated nodes. Given robots' current states and a set of control candidates, the predicted evolution of future system behavior cannot fully determine future network within the planning horizon in a nosiy environment. 
	}
	\vspace{-10pt}
	\label{fig. Task illustration}
\end{figure*}

It is a well-known problem in active planning under uncertainty that the planner cannot perfectly predict future measurements at planning time due to the existence of uncertainties from both motion and measurement processes. As a result, the planner cannot exactly evaluate the objective function given a control candidate. Typically, the measurement-related unknowns during planning include: 

1) The unknown distribution of raw measurement data from perceptive devices if a connection exists, and 

2) A nondeterministic network of future measurements.

While the first is easily understood, the second one arises from the limited capacities of sensors e.g. the range limitation of a radio device or field of view (FOV) of a camera as shown in Fig.~\ref{fig. Task illustration}. Since the future robot states are actually randomly distributed in uncertain scenarios, a measurement event in the future i.e. the establishment of a measurement connection between two future nodes, is a random variable under the presence of limited sensing capacities. Therefore, the network topology of future measurements, which consists of all measurement events during the planning horizon, is nondeterministic at the planning time.

The problem of unknown future measurements during planning has been highlighted and partly relieved in some recent works \cite{platt2010belief,indelman2015planning,pathak2018unified,farhi2019ix,indelman2018cooperative}. However, the state-of-the-art works predict these future unknown variables only by some intuitive or experience-based methods. For example, the maximum likelihood (ML) assumption used in \cite{platt2010belief}, wherein the future raw measurement data is simply determined by the prior estimates and therefore the distribution of future connection is calculated by the Bernoulli model. Besides, \cite{indelman2015planning} theoretically derives the expectation of the objective function over the unknown raw measurement data based on the one-step Gauss-Newton (GN) iteration process. However, the network of future measurement connections is still inherited and computed from the ML assumptions. \cite{pathak2018unified} reasons about the future data association problem in perceptually aliased environments, which is to identify the observed node if a future connection is established. However, once the data association is clear, then the raw measurement value is sampled from the propagated belief and the probability of this future connection is directly derived from a predefined and known Gaussian model. In addition, sampling-based methods and Gaussian models have also been applied in \cite{farhi2019ix,indelman2018cooperative} to approximate these future variables.

The \textbf{key observation} in this work is that the performance of active planning can be improved if we can accurately predict the distributions of these future variables at planning time based on reasonable assumptions. To test such observation, we leave out the calculation of future raw data and concentrate on accurately predicting the distribution of future network connectivity. We mainly focus on a class of range-based communication and observation devices with maximum range threshold. An adaptive power series expansion (APSE) algorithm is developed in this paper to predict the probability of each future measurement connection. 
Both theoretical analysis and numerical results are provided to guarantee the accuracy of APSE. It is then compared with several existing experience-based methods in two active planning scenarios, where robots are deployed to complete high-level tasks while localizing themselves. Numerical results show that an active planning framework using the APSE algorithm achieves more than 50\% reduction over the localization uncertainty
compared to using the Bernoulli model. The main contributions of this note are twofold,

1) APSE: to more accurately compute the exact probability of future connection between two nodes equipped with range-based communication and measurement devices given only current estimates and control candidates. Both theoretical guarantees and statistical validations are presented.

2) Extensive simulation results, showing that a more accurate event likelihood can indeed help improve the active planning performance.

We note that the first contribution is an extension of our previous work presented in \cite{zhang2020connectivity}. Both papers apply the power series expansion formula of the quadratic positive form in a normal distribution from Provost and Mathai's theorem \cite{provost1992quadratic} to predict the probability of a measurement connection, wherein an infinite summation of series is required. In the previous work, we firstly take the finite-term approximation to realize computational tractability and then three more modifications are developed to tackle the truncation error introduced by omitting higher-order terms. As further contributions, this manuscript presents a rigorous theoretical analysis for the truncation error, by which we show that this error can be bounded by an arbitrary threshold if some sufficient conditions hold. These conditions are then summarized into the adaptive principles for the selection of the maximum power degree in the finite-term approximation. Besides, a new expansion method, termed as translational approximate covariance expansion (TRACE) is proposed for the singular cases where the covariance of the relative distance distribution is too small to stabilize the finite-term approximation.

\section{Related Work} 
In this section, works most related to our approaches are discussed. The state of the art of three topics is presented with respect to the developments of belief space planning, connectivity control and cooperative localization. 

\textbf{Belief Space Planning:} Active planning performs sequential online decision making which takes into account all the available information gathered up to the point each new decision is made.

Such problems can be efficiently solved as a Markov Decision Process (MDP) \cite{puterman2014markov} or by Dynamic Programming (DP) methods \cite{denardo2012dynamic}. However, early research mostly only considered systems with deterministic or finite state transition processes, which are feasible to solve in both computational complexity and optimality aspects. 

In contrast, the active planning problem under uncertainty is often formulated under the POMDP framework due to the incorporation of uncertainties arising from both imperfect motion and noisy measurement.
Belief space planning (BSP) methods are exactly an instantiation of a POMDP problem. Instead of planning in a configuration space where knowledge of the mean of the estimated state distribution is sufficient for good performance, BSP optimizes the underlying problem in a continuous \textit{belief space} by additionally taking the uncertainty of the state distribution (or belief, such as the covariance) into consideration. 

Recent developments of BSP mainly focus on approximation methods in order to realize the balance of tractability and run-time complexity. These methods can be generally categorized into discrete and continuous groups according to the type of state space it considers. Discrete-domain planning, including point-based value iteration solvers \cite{shani2013survey,Hsu2008PointTargetTracking} and sampling-based methods \cite{wang2016improved,prentice2009belief,agha2014firm}, perform discretization over the belief space and hence generate a finite set of control candidates from which the optimal strategy is determined according to the maximum probability along each path (or equivalently the minimum uncertainty at the goal). Approaches in continuous spaces, optimized by a linear-quadratic Gaussian framework in belief space \cite{platt2010belief} or gradient descent \cite{indelman2015planning}, derive a locally optimal solution from a given initial plan that is generated from the discretized methods. More recently, some further improvements have contributed to reducing the computational complexity in large-scale deployments by leveraging the similarity between candidate actions for some specific forms of  cost \cite{kopitkov2017no} and the recovery of future posterior covariance \cite{kopitkov2019general}.   

\textbf{Connectivity Control:} The problem of modeling and predicting network connectivity has been widely investigated in the area of wireless communication and \textit{ad hoc} networks \cite{hsieh2008maintaining}. The preservation of connectivity is of great importance for packet routing, resource allocation, and bandwidth management, which have driven interest in the problem of connectivity control within the community. The connectivity of a system is often modeled as a proximity graph and thus the underlying problem is discussed with the help of graph theory. The Fiedler value of a graph, or the second smallest eigenvalue of the Laplacian matrix of a graph, is a concave function of the Laplacian matrix and implies network connectivity when it is positive definite \cite{fiedler1973algebraic}. Therefore, optimization-based connectivity controllers are designed through maximizing or minimizing the Fiedler value in either centralized \cite{kim2005maximizing} or distributed applications \cite{yang2010decentralized,franceschelli2013decentralized}.

Another prevailing methodology are the gradient-based potential field methods which exploit the graph Laplacian matrix to construct a convex potential function and treat the loss of connectivity as obstacles in free space \cite{zavlanos2007potential,zavlanos2011graph}. As the aforementioned theories only consider proximity-based communication models composed of disk-based or uniformly-fading-signal-strength communication links, more complex or realistic configurations have been further investigated to evaluate the effectiveness of multi-path fading, intermittent or recurrent communication constraints, various communication models and so on 
\cite{kantaros2016distributed,banfi2018strategies}.

Besides, recent developments of connectivity research have extended far beyond its basic concept as just a medium of information transmission. The significance of connectivity is no longer only about the routing of information in a network, instead the underlying applications of relative observation as properties of communication links, like the strength of signal (SOS) or time-of-flight (TOF), have been widely introduced for solving navigation and localization problems. 

The differences between the traditional connectivity problem and the measurement event considered in this paper can be distinguished in two aspects. Firstly, despite the great success in connectivity control, most existing approaches rely on perfect knowledge of the sensor state, which is unavailable for real deployments in uncertain environments, like GNSS-limited areas. Here, the future measurement event between two adjacent nodes is stochastic due to the presence of various uncertainties. Secondly and most importantly, the problem of predicting the existence of a future measurement event, which we consider in this paper, only takes the change in system connectivity as an intermediate tool for achieving final objectives, while the traditional connectivity control problem only targets a desired network topology through their designed methods. 

\textbf{Cooperative Localization:} Reliable relative observations enable the promising paradigm of using cooperative localization (CL) to navigate a multi-robot system (MRS) through challenging environments.

In contrast to landmark-based localization methods, e.g. the typical framework of simultaneous localization and mapping (SLAM), CL depends little on the environment representations and hence is of potential advantage to be applied in more critical situations, for example, GNSS-limited areas such as in deep ocean or outer space. 

Early on, the performance of CL has been studied as the problem of sensor deployment through both theoretical analysis
\cite{roumeliotis2003analysis,mourikis2006performance} 
and experimental validation \cite{kurazume2000experimental,trawny2004optimized}. 
Using the centralized EKF as the engine of fusing measurements collected from the network, it has already been shown that, when the absolute position measurement is available i.e. at least one robot can get a GNSS signal or measurements from an anchor, then the CL system is observable and the upper bound of steady-state location uncertainty is constant. The solution quality becomes independent of the initial uncertainty and is only dependent on the topology of the relative position measurement graph, and the accuracy of the proprioceptive and exteroceptive sensors of the robots. Recent research on CL arising from the field of wireless communication takes the exact radio model of the observation process into consideration, simultaneously considering the impacts from multipath propagation, transmitting power, clock delay and so on \cite{shen2010fundamental1,shen2010fundamental2}. Collecting both sensor positions and the parameters of the communication channels into the joint estimated state and based on the Fisher information matrix (FIM), the authors proposed the equivalent FIM (EFIM) by extracting the subset of the FIM corresponding to the position states via the Schur complement operation \cite{cottle1974manifestations}. As a result, the inverse of the trace of the EFIM, which is named the square position error bound (SPEB), can indicate the lower bound of the estimated sensor position error under the current network configuration.

Active planning of both configurations and motion strategies for CL in MRS is more critical and necessary than for the localization of a single robot. The main reason is that relative measurements contribute the only location information source to correct localization uncertainty in MRS. In contrast, there may be several possible options for single robot applications. For example, in active SLAM, different loop closures can be selected and formed by re-observing different landmarks to reduce the uncertainty \cite{cadena2016past}. Research on network optimization includes optimal formation in \cite{hidaka2005optimal}, motion strategies in 
\cite{zhou2011multirobot} 
and sensing frequency in \cite{mourikis2006optimal}. The SPEB indicator-based active operations of the network are investigated in \cite{win2018network} with respect to the problems of node prioritization, sequential node activation, node deployment and power allocation and so on. When uncertainties from both motion and measurement processes are considered, the optimal motion strategies for MRS have been studied in the contexts of target tracking, information gathering \cite{schlotfeldt2018anytime}, active SLAM \cite{cadena2016past}, and autonomous navigation \cite{levine2013information} or coverage \cite{zhong2011distributed}. As a good combination, the BSP framework has been applied to the generation of control inputs of MRS while CL is activated for positioning robots in \cite{regev2016multi,regev2018decentralized}.

\begin{table*}[htbp] 	
	\centering
	\caption{Collection of notations and their definitions.\label{Tab. collNotations}} 	
	\begin{tabular}{cl}
		\hline
		Notation and definition & Descriptions \\ 
		\hline
		$\vect{X}^k:=\{ \vect{p}^k_1,\cdots,\vect{p}^k_N \}$  & Set of all robot states at $k$-th time step\\
		$\vect{X}^{k_1:k_2}:=\{ \vect{X}^{k_1},\cdots,\vect{X}^{k_2} \}$  & Set of all robot states between time $k_1$ and $k_2$ \\
		$\vect{Z}^{k}_{i}:=\{ \vect{z}_{i,j}^k | \forall v_j \in \vect{N}_i^k \} = \{ \vect{\vec{Z}}^k_i, \vect{N}_i^k \}$  & Set of measurements for robot $i$ at time $k$ \\
		$\mathcal{E}^{k}:= \vect{N}_1^k \bigcup \vect{N}_2^k \bigcup \cdots \bigcup \vect{N}_N^k $ & The communication and observation network of system at time $k$ \\	
		$\vect{Z}^k,\vect{Z}^{k_1:k_2},\vect{U}^{k},\vect{U}^{k_1:k_2},\mathcal{E}^{k_1:k_2}$ & The same definitions as $\vect{X}^k$ and $\vect{X}^{k_1:k_2}$\\		
		$\vect{W}:=\{ \vect{W}_1,\cdots,\vect{W}_M \}$ & Set of states for all anchors in environment\\
		$[\vect{XW}]^k:=\{ \vect{X}^k, \vect{W} \}$ & Set of states for all nodes in environment at time $k$ \\ 
		$\vect{X}_j^k \in \vect{[XW]}^k$	&  Single node state that can indicate either robots or anchors \\
		$\mathcal{H}^{0:k}:=\{ \vect{Z}^{0:k}, \vect{U}^{0:k-1}, \vect{W} \}$ & History data that have already been collected by time $k$ \\
		$\mathcal{H}^{0:k+l|k}:=\{ \mathcal{H}^{0:k}, \vect{U}^{k:k+l-1} \}$ & History data with extra control candidates for $l$-th planning step. \\
		$\vect{E}:=\{ \mathcal{E}^{k+1:k+L}_{iter} | \forall iter \in \mathbb{N}^+_{N_{con}} \}$  & Set of all possible future connectivities during planning horizon $L$ \\
		\hline
	\end{tabular} 
	\vspace{-15pt}
\end{table*}

\section{Preliminaries and Problem Formulation }
\subsection{General Notations}

In this paper, the notation $\mathbb{N}^+$ is used for the set of all positive integers. Given an integer $Z>0$, we denote the set of all the positive integers no greater than $Z$ as  $\mathbb{N}^+_{Z}:= \{ 1,2,...,Z \}$. The notation $\mathbb{R}^n$ represents the vector space with dimensionality $n$ while $\mathbb{R}^{n \times m}$ is the matrix space whose element size is $n \times m$. The operator $|\cdot|$ returns the absolute value of a scalar.
The double-colon $\mathbf{:}$ used between two superscripts or subscripts represents a subset of consecutive elements, for example, $\vect{A} \in \mathbb{R}^n := \{a_1,a_2,...,a_n\}$, then we have $\vect{A}_{2:5} = \{a_2,a_3,a_4,a_5\}$.

\subsection{Configurations and Graph Theory} \label{Notations}

Consider a MRS composed of $N$ identical robots operating in a GNSS-limited environment. We further suppose that the movement and measurement processes all suffer from stochastic noise. Let $\vect{p}_i^{k}$ and $\vect{W}_j$ denote the $i$-th robot state at time step $k$ and the $j$-th assistant node state, respectively. All robots can only get their initial states and then need to cooperatively localize themselves using observations collected by the sensors and messages exchanged within the MRS. The assistant node state considered here may describe the positions of either landmarks (which remain static and unknown) or anchors (whose global positions are exactly known), which are observed by robots. While landmarks are generally used in the SLAM literature, anchors are often introduced in cases where specialized robots can occasionally get access to their exact position or when base stations are deployed in the environment broadcasting their global positions to nearby robots.  
Since landmarks can be categorized and formulated as special robots who don't suffer from motion noise and whose control inputs are always zeros, without loss of generality, we assume all assistant nodes are represented by anchors in this paper for simplicity, whose number is denoted by $M$. In the following, we intermittently use `nodes' to represent robots or anchors.

\textbf{\textit{Graph Theory}}: 
Assume all robots and anchors are associated with the nodes of a time-varying network topology graph $\mathcal{G}^k = \{ \mathcal{V}, \mathcal{E}^k \}$ where $\mathcal{V}$ and $\mathcal{E}^k$ are the set of graph vertices and edges at time step $k$, respectively. The $i$-th node in the graph is denoted by $ v_i, i \in \mathbb{N}^+_{M+N}$ as a result $ \mathcal{V} := \{ v_i|i\in \mathbb{N}^+_{M+N}\} $. In addition, the set of vertices $\mathcal{V}$ can be further divided into the subset of robots $\mathcal{V}_R$, and anchors $\mathcal{V}_A$, where $\mathcal{V}_R \cup \mathcal{V}_A = \mathcal{V}$. An edge between nodes $v_i$ and $v_j$ at time step $k$ is denoted by $ s_{i,j}^k=1$ if they can mutually observe and communicate with each other, and $s_{i,j}^k=0$ if not. Therefore, the sub-graph among robots is undirected, i.e. $s_{i,j}^k = s_{j,i}^k, \forall i,j \in \mathcal{V}_R$. However, the edge only exists from robots to anchor, and not the other way around. The edge set is then defined as $\mathcal{E}^k:=\{ v_i \times v_j \in \mathcal{V}_R \times \mathcal{V} | s_{i,j}^k = 1 \}$. The neighboring set $\vect{N}_i^k$ includes all the nodes $v_j$ such that $s_{i,j}^k = 1$ i.e. $\vect{N}_i^k := \{ v_j | s_{i,j}^k = 1 \}$ and its size is denoted by $n_i^k$. In the following, an edge $s_{ij}^k$ and the edge set $\mathcal{E}^k$ would be also refered to as the \textbf{connection} and the \textbf{network topology}, respectively. Furthermore, the node $v_i$ is simply denoted by its index $i$ and $i \in \mathcal{V}$ can implicitly denote all nodes in the environment, including both robots and anchors. Given the aforementioned scenario, if the planning horizon is $L$, then the number of future possible network topologies $N_{con}$ can be calculated by $N_{con} = 2^{N_s}, N_s = \frac{L(M+N)(M+N-1)}{2}$.

\subsection{Problem Formulations and Notations}
Given the notations in Table~\ref{Tab. collNotations}, we consider the conventional state transition model with additive Gaussian noise:
\begin{equation} \label{eq. motion model for robots}
	\vect{p}_i^{k+1} = f_i \left[ \vect{p}_i^k,\vect{u}_i^k,\vect{w}_i \right],
\end{equation}
where $\vect{w}_i \sim \mathcal{N}(\vect{0}, \vect{R})$ with known information matrices $\vect{R}$. 
We further denote the corresponding probabilistic term of \eqref{eq. motion model for robots} as $p(\vect{p}_{i}^{k+1}|\vect{p}_i^k,\vect{u}_i^k)$.

The observation model under investigation is slightly different from the prevailing research in the presence of limited sensing capacity. We instead explicitly denote the output of the observation model by a 2-tuple $\vect{z}_{ij}^k:=<\vect{\vec{z}}_{ij}^k,s_{ij}^k>$. As earlier defined in the \textbf{\textit{Graph Theory}} part, $s_{ij}^k$ is a binomial random variable representing the existence of a connection between node $i$ and $j$ at time step $k$. $\vect{\vec{z}}_{ij}^k$ is the raw measurement data from a sensing device and is only valid when $s_{ij}^k = 1$, i.e.  
\begin{equation} \label{eq. observation model for robots}
	\vect{\vec{z}}_{i,j}^k = h_{ij} \left[ \vect{p}^k_i,\vect{X}^k_j, \vect{v}_{i,j} \right], 
\end{equation}
where $\vect{v}_{i,j} \sim \mathcal{N}(\vect{0}, \vect{Q})$ with known information matrices $\vect{Q}$.
The dimensionality of measurement $\vect{\vec{z}}_{i,j}$ is determined by the type of sensor. For example, $\vec{z}_{i,j}$ is a scalar if only the relative distance is observed by, for example, the radio sensor through measuring the TOF or SOS. And $\vect{\vec{z}}_{i,j} \in \mathbb{R}^{2}$ if both distance and bearing are measured in radar. 

Taking as an example the measurement of relative distance, which can be realized by e.g. UWB, ultrasonic sensors, etc. If the real distance between any two nodes $i \in \mathcal{V}_R, j \in \mathcal{V}$ is:
	\begin{equation} \label{eq. regualtion function 2}
		d_{i,j}(k) = \left\lVert \vect{p}_{i}^k - \vect{X}_{j}^k \right\rVert_2,
	\end{equation}
	then given the maximum sensing radius $\rho$, the connection variable has:
	\begin{equation}	\label{eq. regualtion function 01}
		s_{i,j}^k = \left\lbrace
		\begin{array}{ll}
			1, & d_{ij}(k) \le \rho \\
			0, & d_{ij}(k) > \rho, 		
		\end{array}			
		\right.
	\end{equation}
	and the raw measurement data has:
	\begin{equation} \label{eq. observation model for robots_realtiveDistance}
		\vect{\vec{z}}_{i,j}^k = d_{ij}(k) + v_{ij}, \quad v_{ij} \sim \mathcal{N}(0,Q).
	\end{equation}
	We remark that the disk model is fairly standard in ad hoc networks where all sensors are identical and of the same transmission power. Therefore, this model is commonly applied to represent range-based measurement and communication devices.
	Thus, the probabilistic term for the observation model $p(\vect{z}_{i,j}^k|\vect{p}_{i}^k,\vect{X}_j^k)$ can be decomposed as:
	\begin{equation} \label{eq. decomMea}
		\begin{aligned}
			p(\vect{z}_{i,j}^k|\vect{p}_{i}^k,\vect{X}_j^k) &= p(\vect{\vec{z}}_{i,j}^k,s_{ij}^k|\vect{p}_{i}^k,\vect{X}_j^k) \\
			&= p(\vect{\vec{z}}_{i,j}^k|\vect{p}_{i}^k,\vect{X}_j^k,s_{ij}^k = 1)  p(s_{ij}^k = 1|\vect{p}_{i}^k,\vect{X}_j^k).
		\end{aligned}
\end{equation}

We can write the posterior probability distribution function (pdf) over the joint state as, 
\begin{equation}
	p(\vect{X}^{0:k}|\mathcal{H}^{0:k}).
\end{equation} 

Furthermore, given the prior distribution of robot states $p(\vect{p}_i^0),\forall i \in \mathcal{V}_R$, we can recursively expand the joint belief over system states by,
\begin{equation} \label{eq. belief expansion at current step}
	\begin{aligned}
		&b(\vect{X}^{0:k}) = p(\vect{X}^{0:k}|\mathcal{H}^{0:k}) \\
		&= \prod_{i \in \mathcal{V}_R} p(\vect{p}_i^0) \prod_{t=1}^{k} \left[ p(\vect{p}_{i}^{t}|\vect{p}_i^{t-1},\vect{u}_i^{t-1}) \prod_{j\in \vect{N}_i^t} p(z_{i,j}^t|\vect{p}_i^t,\vect{X}_j^t) \right].
	\end{aligned}
\end{equation}

If Gaussian noise is assumed, the belief can simply be denoted by a vector $\vect{\bar{X}}^{0:k}$ and a covariance (information) matrix $\vect{\Sigma}_{\bar{X}^{0:k}}^k$ representing the estimated means and uncertainties, respectively, i.e. $b(\vect{X}^{0:k})   =  \mathcal{N}(\vect{\bar{X}}^{0:k},\vect{\Sigma}_{\bar{X}^{0:k}}^k).$
In the estimation problem, the final purpose is to derive the optimal values of $\vect{\bar{X}}^{0:k}$ and $\vect{\Sigma}_{\bar{X}^{0:k}}^k$ using the available measurements i.e. $\vect{Z}^{0:k}$ that has been collected from all sensors up to time $k$. Therefore the connection variable in the estimation problem is deterministic and can be directly extracted from the collected data $\vect{Z}^{0:k}$ by:
\begin{equation}	\label{eq. connectionVarInEstimation}
	\left\lbrace
	\begin{array}{ll}
		s_{i,j}^t = 1, p(s_{i,j}^t = 1) = 1 & \text{if } \vec{z}_{ij}^t \in \vect{Z}^{0:k}\\
		s_{i,j}^t = 0, p(s_{i,j}^t = 1) = 0 & \text{otherwise}, 		
	\end{array}			
	\right. \forall t \le k,
\end{equation}
and the decomposition in \eqref{eq. decomMea} can be simplified as:
$$p(z_{i,j}^t|\vect{p}_i^t,\vect{X}_j^t) = p(\vec{z}_{i,j}^t|\vect{p}_i^t,\vect{X}_j^t), \forall \vec{z}_{i,j}^t \in \vect{Z}^{0:k}.$$
However, such simplification cannot be applied to the planning process since it lacks future measurement data. To highlight its impact on active planning, let's consider an expectation-based objective function $J_k$ as:
\begin{equation} \label{eq. objective function}
	\begin{aligned}
		J_{k}(\vect{U}^{k:k+L-1}) &= \mathop{\mathbb{E}}_{\vect{Z}^{k+1:k+L} }  \left[ \sum_{l=1}^{L} c^{l} \left( b(\vect{X}^{0:k+l}), \vect{U}^{k:k+l-1} \right)  \right],
	\end{aligned}	
\end{equation}
where the expectation operator $\mathbb{E}$ is taken with respect to the future observations of all robots i.e. $\vect{Z}^{k+1:k+L}$. The reason for using the expectation here is that the future measurements are stochastic in the presence of uncertainties from both motion and observation noise. $c^l(\cdot)$ denotes an immediate cost function at the $l$-th look-ahead step in terms of the joint belief $b(\vect{X}^{k+1:k+l})$ and of the given control candidates. Traditionally, a motion policy from time $k$ to $k+L$ is taken according to:
\begin{equation} \label{eq. original motion policy}
	\begin{aligned}
		^{*}\vect{U}^{k:k+L-1} &= \{ ^{*}\vect{U}^{k},^{*}\vect{U}^{k+1},...,^{*}\vect{U}^{k+L-1} \} \\
		&= \mathop{\arg\min}_{\vect{U}^{k:k+L-1} } \quad J_{k}(\vect{U}^{k:k+L-1}).
	\end{aligned}
\end{equation}

Now, let's rewrite the expectation in \eqref{eq. objective function} explicitly by recalling the composition of $\vect{Z}^{k+1:k+L} = \{ \vect{\vec{Z}}^{k+1:k+L}, \mathcal{E}^{k+1:k+L} \}$ and applying the total probability over the future network and raw measurement data: 
\begin{equation} \label{eq. objective function1}
	\begin{aligned}
		& J_{k}(\vect{U}^{k:k+L-1}) =  \\
		& \int_{\vect{\vec{Z}}^{k+1:k+L}} \int_{\mathcal{E}^{k+1:k+L}} p( \vect{\vec{Z}}^{k+1:k+L}, \mathcal{E}^{k+1:k+L} |\mathcal{H}^{0:k+L|k} ) \\
		& \left[ \sum_{l=1}^{L} c^{l} \left( \underbrace{p(\vect{X}^{0:k+l}|\mathcal{H}^{0:k+l|k}, \vect{\vec{Z}}^{k+1:k+l}, \mathcal{E}^{k+1:k+l}) }_{\text{$b(\vect{X}^{0:k+l})$}}, \vect{U}^{k:k+l-1} \right)  \right] \\
		& =  \sum_{\mathcal{E}^{k+1:k+L}_{iter} \in \vect{E} } \underbrace{ p( \mathcal{E}^{k+1:k+L}_{iter} |\mathcal{H}^{0:k+L|k} ) }_{\text{term A}}  \\
		&  \int_{\vect{\vec{Z}}^{k+1:k+L}}  \underbrace{ p( \vect{\vec{Z}}^{k+1:k+L} |\mathcal{H}^{0:k+L|k},\mathcal{E}^{k+1:k+L}_{iter} )}_{\text{term B}} \left[ \sum_{l=1}^{L} \underbrace{c^{l} \left( \cdot \right)}_\text{term C} \right]. 
	\end{aligned}	
\end{equation}

The objective is composed of three terms. Term A is defined as network probability, which indicates how likely it is to form such a network during the planning horizon given the history data and the control candidates. Term B represents the likelihood of collecting a possible measurement from sensors when the network is specified. \blue{Remark that term B can be computed by the expansion in the estimation problem in \eqref{eq. belief expansion at current step} given the existence of future network. Then the probability of raw measurement along each connection can be calculated by the measurment model as defined in \eqref{eq. observation model for robots_realtiveDistance}. } Term C is the immediate cost corresponding to the posterior probability of joint states from $0$ to $k+l$ given a set of raw measurements and the network topology. \blue{It should at least concerns with the predicted uncertainty of future robot states. }

It's clear that the network probability plays the role of weighting how much influence each future network branch can contribute to the objective function. Therefore, it is vital for the evaluation of the planning objective to accurately predict the network probability, instead of just using rough approximations as in existing works.  

\subsection{Spatial and Temporal Independence Assumption }

Regarding the network probability (Term A), we can proceed by marginalizing over all possible robot states and applying the chain rule, yielding,
	\begin{equation} 
		\begin{aligned}
			& p( \mathcal{E}^{k+1:k+L}_{iter} |\mathcal{H}^{0:k+L|k} ) = \int_{\vect{X}^{k+1:k+L}}  \\ & p(\mathcal{E}^{k+1:k+L}_{iter}|\vect{X}^{k+1:k+L},\mathcal{H}^{0:k+L|k} )  p(\vect{X}^{k+1:k+L}|\mathcal{H}^{0:k+L|k}).
		\end{aligned}	
\end{equation}
\blue{The following takes the assumption} that the measurement connections in future networks  $\mathcal{E}^{k+1:k+L}$ are both temporally and spatially independent. 
\begin{align}
		p(\mathcal{E}^{k}) = \prod_{i \in \mathcal{V}_R} \prod_{j \in \vect{N_i^k}} p(s_{i,j}^k), \label{eq. spatiallyIndependenceFutureC} \\
		p(\mathcal{E}^{k:k+L}) = \prod_{t=l}^{L} p(\mathcal{E}^{k+t}). \label{eq. temporallyIndependenceFutureC}
\end{align}

\blue{Note that this is a common assumption in the field of planning under uncertainty \cite{indelman2015planning,indelman2018cooperative}. Such assumption originates from estimation problems} where the Markov assumption is widely applied such that the future does not depend on the past given the present. It defines the merit of completeness to the states of robots, which entails that knowledge of past states, measurements, or controls carry no additional information that would help us to predict the future more accurately. As the states are complete under this assumption, then the measurement at the current time is conditionally independent of the past and future robot states, i.e. $[\vect{XW}]^k$ is sufficient to predict the potentially noisy measurement $\vect{Z}^{k}$.\textbf{ In other words, the Markov assumption guarantees that measurements are temporally independent \cite{thrun2002probabilistic}.} 
Furthermore, the measurements are also spatially independent \blue{if the basic functioning principle of the observation devices on each robot don't interfere with each other. The extension of such an assumption from estimation to planning hasn't been fully discussed in existing works. As this problem is beyond the scope of this paper, we just take this assumption and} leave the problem of investigating the exact relationship of the measurement connections during the planning session as an open problem for future research.

Based on the aforementioned assumption, the network probability is also independent from the history of data given the distribution of future states, which yields:
	\begin{equation} \label{eq. NeTopolyProba}
		\begin{aligned}
			& p( \mathcal{E}^{k+1:k+L}_{iter} |\mathcal{H}^{0:k+L|k} ) \\
			& = \int_{\vect{X}^{k+1:k+L}} p(\mathcal{E}^{k+1:k+L}_{iter}|\vect{X}^{k+1:k+L}) p(\vect{X}^{k+1:k+L}|\mathcal{H}^{0:k+L|k}) \\
			& = \prod_{\forall s_{ij}^{t} \in \mathcal{E}^{k+1:k+L}} \int_{\vect{p}_i^{t}, \vect{X}^{t}_j}   \left[  p( s_{ij}^{t} | \vect{p}_i^{t}, \vect{X}^{t}_j ) p(\vect{p}_i^{t}, \vect{X}^{t}_j | \mathcal{H}^{0:k+L|k} )  \right].
		\end{aligned}	
	\end{equation}
As a result, the network probability can be partitioned according to every possible individual connection. The integral over future states actually indicates that the probability computation of a connection is conditioned on the full distribution of its adjacent nodes. Since the integral can be over any distribution of the joint states $\vect{X}^{k+1:k+L}$, similarly to \cite{pathak2018unified}, we use the propagated belief $b(\vect{X}^{k+1:k+L|k}) $ in the following sections as:
\begin{align}
	b(\vect{X}^{k+1:k+L|k}) &= \int_{ ^{\neg} \vect{X}^{k+1:k+L}} b(\vect{X}^{0:k+L|k}) \\
	b(\vect{X}^{0:k+L|k}) &= b(\vect{X}^{0:k}) p(\vect{X}^{k+1:k+L}|\vect{X}^k,\vect{U}^{k+1:k+L}).	
\end{align}

Therefore, the remaining problem is: 

\noindent\textbf{Connection probability problem (CPP):} For a system equipped with range-based measurement devices and assuming Gaussian noise, accurately compute the probability of each future connection $s_{ij}^{k+t}, \forall i,j \in \mathcal{V}, \forall t \in \mathbb{N}^+_{L}$, i.e. $p(s_{i,j}^{k+t}=1)$, given the control candidate $\vect{U}^{k:k+L-1}$, current estimates $b(\vect{X}^{0:k})$ and prior knowledge about the sensing principles in \eqref{eq. regualtion function 2}-\eqref{eq. observation model for robots_realtiveDistance}. 

\section{APSE: Probability of a Connection}\label{Section 4: Connectivity prediction for two nodes}

In this section, an Adaptive Power Series Expansion (APSE) algorithm is developed for the CPP. First, a basic lemma about the factorial of a positive integer $n \in \mathbb{N}^+$ is presented.

\begin{lemma} \label{lm. lemma of the factorial of n}
	Given a positive integer $n$, its factorial has $ \left( \frac{n}{e} \right)^n < n! < e \left( \frac{n}{2} \right)^n$, where $e$ is the Natural Constant.
\end{lemma} 

Further, Provost and Mathai's \blue{lemma}  forms the basis of our main algorithm. \blue{We state the theorem below and refer interested readers to Section 4.2 pp. 91-99 in \cite{provost1992quadratic} for the complete proof.}

\begin{lemma} \label{Theorem: Power series expansion of quadratic form of normal distribution}
	Given the p-dimensional multivariate normal distribution $\vect{X} \sim \mathcal{N}_p(\vect{\mu}, \vect{\Sigma})$, $\vect{\Sigma} > 0$ and its quadratic form $Y = Q(\vect{X}) = \vect{X}^T\vect{A}\vect{X}$,$\vect{A} = \vect{A}^T > 0$, let $\vect{b} = \vect{P}^T\vect{\Sigma}^{-\frac{1}{2}}\vect{\mu}$ and $\vect{\lambda} = [\lambda_1,...,\lambda_p]^T$ be the eigenvalues of $ \vect{\Sigma}^{\frac{1}{2}}A\vect{\Sigma}^{\frac{1}{2}}$, i.e. $\vect{P}^T\vect{\Sigma}^{\frac{1}{2}}A\vect{\Sigma}^{\frac{1}{2}}\vect{P} = diag(\vect{\lambda})$, $\vect{P}^T\vect{P} = I$. Then the corresponding cumulative distribution function (CDF) of $Y$, that is $p\{ Y \le y \}$, will be denoted by $F_p(\vect{\lambda};\vect{b};y)$ \blue{and it} can be expanded as follows:
	\begin{equation} \label{eq. PSE cumulative probability function, origin }
		F_p(\vect{\lambda};\vect{b};y) = \sum_{\blue{w=0}}^{\infty}(-1)^w c_w \frac{y^{\frac{p}{2} + \blue{w}}}{\Gamma(\frac{p}{2}+w+1)}, \quad 0<y<\infty,
	\end{equation}
	where the coefficients $c_w$ is defined by,
	\begin{equation} \label{eq. computation of the coefficient c_k}
		\begin{aligned}
			c_0 &= exp(-\frac{1}{2}\sum_{j=1}^{p}b_j^2)\prod_{j=1}^p(2\lambda_j)^{-\frac{1}{2}}, \\
			c_w &= \frac{1}{w}\sum_{r=0}^{\blue{w}-1}d_{w-r}c_r, \quad w \ge 1,
		\end{aligned}
	\end{equation}
	and $d_w$ is given by,
	\begin{equation}\label{eq. computation of the coefficient d_k}
		d_w = \frac{1}{2} \sum_{j=1}^p (1-wb_j^2)(2\lambda_j)^{-w},\quad w \ge 1.
	\end{equation}
\end{lemma}

\blue{More in-depth theoretical analysis of computing the CDF of quadratic form in normal variables by series representations can be found in \cite{kotz1967series1,kotz1967series2}}

\subsection{Modeling and Finite-term Approximation} \label{ModelFiniteApproximation}

At planning time step $k$, for two arbitrary nodes in the scenario, we can extract their state distribution from the propagated belief $b(\vect{X}^{k+1:k+L|k})$. Let's denote their true positions at time step $k+l$ under the control candidate as $\vect{p}_1,\vect{p}_2$. Then if the estimator used for propagated belief $b(\vect{X}^{k+1:k+L|k})$ is complete, the distributions of $\vect{p}_1,\vect{p}_2$ can be reasonably treated by two normal distributions $\bar{\vect{P}}_1 \sim \mathcal{N}_1(\vect{\mu}_1 , \vect{\Sigma}_1)$ and  $\bar{\vect{P}}_2 \sim \mathcal{N}_2(\vect{\mu}_2 , \vect{\Sigma}_2)$, which can be extracted from the propagated belief $b(\vect{X}^{k+l})$. 

As a result, we know that the subtraction of $\bar{\vect{P}}_1$ and $\bar{\vect{P}}_2$ is also a normal distribution. If the two distributions of future robot positions are independent of each other, then subtraction has $\Delta \bar{\vect{P}} = \bar{\vect{P}}_1 - \bar{\vect{P}}_2 \sim \mathcal{N}_{\Delta}(\vect{\mu}_1 - \vect{\mu}_2 , \vect{\Sigma}_1 + \vect{\Sigma}_2)$. Otherwise, let us denote their covariance $\vect{\Sigma}_{1,2}$, then the variance of subtraction is slightly different:  $\Delta \bar{\vect{P}} = \bar{\vect{P}}_1 - \bar{\vect{P}}_2 \sim \mathcal{N}_{\Delta}(\vect{\mu}_1 - \vect{\mu}_2 , \vect{\Sigma}_1 + \vect{\Sigma}_2 - 2\vect{\Sigma}_{1,2})$.

By denoting $\Delta \bar{\vect{P}} = [\Delta \bar{P}_x, \Delta \bar{P}_y]^T$, the square of the true distance $d_{12} = \left\lVert\vect{p}_1 - \vect{p}_2\right\rVert_2$ can be equally represented by a quadratic random variable $Y = Q(\Delta \bar{\vect{P}}) = \Delta \bar{P}_x ^2 + \Delta \bar{P}_y ^2 = \Delta \bar{\vect{P}}^T \vect{A} \Delta \bar{\vect{P}}$, where $\vect{A} = \vect{I}$.

Considering the future connection variable $s_{1,2}$ under a range-based communication and observation device, whose maximum range is $\rho$, the relationship between random variables $Y$ and $s_{1,2}$ is,
\begin{equation}
	p\left( s_{1,2}=1 \right) = p\left( Y \le \rho ^2\right).
\end{equation}

As a result, it is straightforward to use Theorem~\ref{Theorem: Power series expansion of quadratic form of normal distribution} with dimensionality $p=2$ and input $y=\rho^2$ to predict the probability of connection between two nodes, given the estimates of the means and covariance matrices of both nodes. However, Theorem~\ref{Theorem: Power series expansion of quadratic form of normal distribution} needs to compute an infinite summation of power series terms, which is intractable for practical applications. One instinctive approach is to sum only a finite but sufficient number of terms. This means replacing \eqref{eq. PSE cumulative probability function, origin } with,
\begin{equation} \label{eq. PSE cumulative probability function, finite terms }
	F_2(\vect{\lambda};\vect{b};y) = \sum_{w=0}^{w_m}(-1)^w c_w \frac{y^{w+1}}{\Gamma(w+2)}, \quad 0<y<\infty,
\end{equation}
where $w_m$ is a proper maximum degree that our algorithm must determine. However, this method will introduce errors to the calculation of $F_p$ compared with the original infinite sum version in Theorem~\ref{Theorem: Power series expansion of quadratic form of normal distribution}. Thus, the main difficulty that remains is how to maintain the balance between the tractability and probability accuracy of our algorithm. \blue{This trade-off is determined by the choice of degree $w_m$, which will be the main focus of the following subsections.}

\subsection{Stability Analysis} \label{subsec. Stability Analysis}
In this subsection, some theoretical analyses are discussed in terms of the truncation error to provide some insights on how to determine this key parameter $w_m$. After selecting an appropriate $w_m$, the error of the CDF introduced by the omission of higher order terms is,  
\begin{equation} \label{eq. PSE erro, cumulative probability function}
	\Delta F_2(\vect{\lambda};\vect{b};y) = \sum_{w=w_m}^{\infty}(-1)^w c_w \frac{y^{w+1}}{\Gamma(w+2)}, \quad 0<y<\infty.
\end{equation}
Since $w+2$ is always a positive integer, the Gamma function used here is the factorial of $w+1$, i.e.
\begin{equation} \label{eq. truncation error 1}
	\begin{aligned}
		\Delta F_2(\vect{\lambda};\vect{b};y) = \sum_{w=w_m}^{\infty}(-1)^w c_w \frac{y^{w+1}}{(w+1)!} .
	\end{aligned}
\end{equation}

As we can see from the right-hand side of \eqref{eq. truncation error 1}, this error term is also an infinite summation of series. Nevertheless, we are going to show that the error can be arbitrarily reduced so long as the degree $w_m$ is appropriately selected. For the simplicity of derivation, firstly, the coefficient $d_w$ from Theorem~
\ref{Theorem: Power series expansion of quadratic form of normal distribution} is shown to be bounded in the following corollary. 

\begin{cor} \label{Corollary of d_k}
	Given the 2-dimensional multivariate normal distribution $\vect{X} \sim N_2(\vect{\mu}, \vect{\Sigma})$, let $\vect{\lambda} = [\lambda_1, \lambda_2]^T$ be the eigenvalues of covariance $\vect{\Sigma}$ and $\vect{b} = \vect{P}^T\Sigma^{\frac{1}{2}}\vect{\mu} = [b_1,b_2]^T$ where $\vect{P}$ is the eigenvector matrix of $\vect{\Sigma}$. \blue{Then} if we have,
	\begin{equation}
		\left \lbrace
		\begin{aligned}
			&2\lambda_j > 1  \\ 
			&w \ge \frac{1}{b_j^2} + \frac{1}{2\lambda_j - 1} \\
		\end{aligned}	
		\right.
		,\forall \blue{j \in \{1,2\}},	
	\end{equation}
	the limit of series $d_w$ computed according to \eqref{eq. computation of the coefficient d_k} is finite, i.e.
	\begin{equation}
		\mathop{\lim}_{\blue{w\rightarrow \infty}} d_w = 0,
	\end{equation}
	and we have a positive number $d^u > 0$ such that, 
	$$|d_w|<d^u,\forall w \in \mathbb{N}^+.$$
	
\end{cor}
\begin{proof}
	The proof is given in Appendix \ref{subsec proof d_k}.
\end{proof}

Hereafter, we create two new series $\tilde{c}_w$ and $\tilde{d}_w$, respectively, where $\tilde{d}_w$ is defined as, 
\begin{equation} \label{eq. new tilde_d_k}
	\tilde{d}_w := [d^u, d^u, ... ,d^u].
\end{equation}
Let $\tilde{c}_0 = c_0 $, the remaining terms of $\tilde{c}_w$ are again recursively computed by \eqref{eq. computation of the coefficient c_k}, i.e.
\begin{equation} \label{eq. new tilde_c_k}
	\begin{aligned}
		\tilde{c}_0 &= c_0 = exp(-\frac{1}{2}\sum_{j=1}^{p}b_j^2)\prod_{j=1}^p(2\lambda_j)^{-\frac{1}{2}}, \\
		\tilde{c}_w &= \frac{1}{w}\sum_{r=0}^{w-1} \tilde{d}_{w-r}\tilde{c}_r, \quad w \ge 1.
	\end{aligned}
\end{equation}

The idea behind this new series $\tilde{c}_w$ is to serve as an envelope of $c_w$. We note that the definition equation in \eqref{eq. new tilde_c_k} is implicit and vague to understand. Hence, an exact formula is developed in the following corollary to provide a recursive version of its computation and the relationship between $c_w$ and $\tilde{c}_w$.

\begin{cor} \label{corollary of tilde_c_k}
    The new series $\tilde{c}_w$ defined in \eqref{eq. new tilde_c_k} and \eqref{eq. new tilde_d_k}, can be recursively computed by,
	\begin{equation} \label{eq. recurrsive equation of c_k+1 and c_k given a equal d_k}
		\tilde{c}_{w+1} = \frac{d^u+w}{w+1} \tilde{c}_{w},
	\end{equation}
	and it is an envelope of $c_w$ in \eqref{eq. computation of the coefficient c_k}, i.e. 
	\begin{equation} \label{eq. abstract relationship of c_k and tilde_c_k}
		|c_w| \le |\tilde{c}_w|, \forall w \in \mathbb{N}^+.
	\end{equation}
\end{cor}

\begin{proof}
	The proof is given in Appendix \ref{subsec proof c_k}.
\end{proof}

\blue{Moving on, let's denote a positive integer $D \in \mathbb{N}^+$ such that $d^u \le D \le d^u+1$, } then we further introduce a new series defined as,

1) if $D = 1$
\begin{equation} \label{eq. new series e_k^D, D = 1}
	e^1_w = 	\frac{c_w}{(w+1) }, \forall w \in \mathbb{N}^+,
\end{equation}

2) if $D \ge 2$
\begin{equation} \label{eq. new series e_k^D, D>2}
	e^D_w = \left\lbrace 
	\begin{array}{ll}
		c_w, & w \le (D - 2), \\
		\frac{c_w}{\prod_{j=1}^{D} (w+2-j) }, & w \ge (D - 1).
	\end{array}	\right.
\end{equation}
A fact about $e^D_w$ is summarized in the following Corollary.

\begin{cor}\label{corollary of e_k^D}
	Given a finite and positive integer $D$, the limit of $|e_w^D|$ is also finite i.e. 
	\begin{equation} \label{eq. limitation of e_k^D}
		\mathop{\lim}_{w \rightarrow \infty} |e_w^D| = \mathit{Const.} ,
	\end{equation}
	and its maximum is constrained by $|e_w^D| \le \tilde{c}_{D-1}$.
\end{cor}
\begin{proof}
	The proof is given in Appendix \ref{subsec proof e_k^D}.
\end{proof}

Finally, we can conclude our main results of the truncation error in  \eqref{eq. truncation error 1} in the following Theorem.

\begin{theorem} \label{Th. boundedness of truncation error}
	Given the 2-dimensional multivariate normal distribution $\vect{X} \sim N_2(\vect{\mu}, \vect{\Sigma})$, let $\vect{\lambda} = [\lambda_1, \lambda_2]^T$ be the eigenvalues of covariance $\vect{\Sigma}$ and $\vect{b} = \vect{P}^T\Sigma^{\frac{1}{2}}\vect{\mu} = [b_1,b_2]^T$ where $\vect{P}$ is the orthogonal matrix composed of the eigenvectors of $\vect{\Sigma}$. Let us compute the series $d_w$ according to \eqref{eq. computation of the coefficient d_k} and denote a positive integer $D$ \blue{satisfying $d_w \le D \le d_w + 1$}. Suppose, 
	\begin{equation}  \label{eq. sufficient and necessary conditions of in bounded theorem of lamda}
		2\lambda_j >1,\forall j \in \{1,2\}, 
	\end{equation}
	then the truncation error $\Delta F_2(\vect{\lambda};\vect{b},y)$ in \eqref{eq. truncation error 1} can be reduced to a given threshold $\delta f > 0$, i.e. 
	\begin{equation}
		\Delta F_2(\vect{\lambda};\vect{b},y) < \delta f,
	\end{equation}
	if the degree $w_m \in \mathbb{N}^+$ in \eqref{eq. PSE cumulative probability function, finite terms } is selected as,	
	\begin{equation}  \label{eq. sufficient and necessary conditions of in bounded theorem of k_m}
		\left\lbrace
		\begin{aligned}
			&w_m \ge \frac{1}{b_j^2} + \frac{1}{2\lambda_j - 1}, \forall j \in \{ 1,2 \}, \\
			&w_m \ge e^2y + D, \\
			&w_m \ge 3D - ln \left( \frac{\delta f}{c_0y^D} \right) - 1.	 	
		\end{aligned}
		\right.		
	\end{equation}
\end{theorem}

\begin{proof}
	Based on Corollaries \ref{Corollary of d_k}\---\ref{corollary of e_k^D}, we recall and expand the expression $\Delta F_2(\vect{\lambda};\vect{b},y)$ deriving,
	\begin{equation}
		\begin{aligned}
			\Delta &F_2(\vect{\lambda};\vect{b};y) = \sum_{w=w_m}^{\infty}(-1)^w c_w \frac{y^{w+1}}{(w+1)!} \\
			&\le  \sum_{w=w_m}^{\infty} \frac{|c_w|}{\prod_{j=1}^{D}(w+2-j)} \frac{y^{w+1}}{(w+1-D)!} \\
			&\le \sum_{w=w_m}^{\infty} |\tilde{e}^D_w|\frac{y^{w+1}}{(w+1-D)!} \\
			&\le \tilde{c}_{D-1} \frac{y^{w_m+1}}{(w_m+1-D)!} \left(1 + \sum_{w=1}^{\infty} \frac{y^{w}}{\prod_{j=1}^{w} (w_m + 1+ j -D) }  \right) \\
			&< \tilde{c}_{D-1} \frac{y^{w_m+1}}{(w_m+1-D)!} \left(1 + \sum_{w=1}^{\infty} \frac{y^{w}}{ (w_m -D)^w }  \right).
		\end{aligned}
	\end{equation}
	
	On the one hand, the latest infinite summation term in $\Delta F_2(\vect{\lambda};\vect{b},y)$ can be scaled down to the Basel Problem. Let, 
	\begin{equation} \label{eq. Basel equation 1}
		\frac{1}{ \left( \frac{w_m-D}{y} \right)^w } \le \frac{1}{w^2},
	\end{equation}
	then the infinite summation term in $\Delta F_2(\vect{\lambda};\vect{b};y)$ has,
	\begin{equation}
		\left(1 + \sum_{w=1}^{\infty} \frac{y^{w}}{ (w_m -D)^w }  \right) \le \left( 1 + \sum_{w=1}^{\infty} \frac{1}{w ^2} \right)  = 1 + \frac{\pi^2}{6}.
	\end{equation}
	Therefore, the requirement in \eqref{eq. Basel equation 1} can be cast into,
	\begin{equation}
		\begin{aligned}
			\frac{1}{ \left( \frac{w_m-D}{y} \right)^w } \le \frac{1}{w^2} 			& \Rightarrow  \left( \frac{w_m-D}{y} \right)^k \ge w^2 \\
			& \Rightarrow  \ln(w_m-D) - \ln(y) \ge \frac{2\ln(w)}{w}. 	
		\end{aligned}
	\end{equation}
	It's clear that the maximum of $\frac{\ln(w)}{w}$ is $\frac{1}{e}$. Accordingly we need,
	\begin{equation} \label{eq. condition of Basel Problem}
		\begin{aligned}
			\ln(w_m-D) - \ln(y) \ge \frac{2}{e} 
			\Rightarrow   w_m \ge e^{\frac{2}{e}}y + D.
		\end{aligned}
	\end{equation}
	
	On the other hand, let's denote $g(w_m) = \frac{y^{w_m+1}}{(w_m+1-D)!}$ and take a subtraction between $g(w_m)$ and $g(w_m+1)$, which yields,
	$$
	\begin{aligned}
		& g(w_m+1) - g(w_m)\\ 
		&= \frac{y^{w_m+2}}{(w_m+2-D)!} - \frac{y^{w_m+1}}{(w_m+1-D)!}\\
		& = \frac{y^{w_m+1}}{(w_m+1-D)!} \left( \frac{y}{w_m+2-D} - 1\right).
	\end{aligned}
	$$
	As a result, we know that $g(w_m)$ takes its maximum when $w_m=y+D-2$ or $w_m=y+D-1$, and that $g(w_m)$ is decreasing after $w_m > y+D-1$. Besides, its limit is exactly zero when the selected $w_m$ tends to infinity.
	
	Consequently, the error function is bounded by 
	\begin{equation}
		\begin{aligned}
			\Delta &F_2(\vect{\lambda};\vect{b};y) 
			&< \tilde{c}_{D-1} \left( 1+\frac{\pi^2}{6} \right)g(w_m). 
		\end{aligned}
	\end{equation}
	
\begin{table}[htbp]
    \vspace{-10pt}
	\centering
	\caption{\blue{Values of $w_m$ when $g(w_m)$ is less than a certain threshold given finite $y$ and $D$. \label{T1. table of k_m}}}
	\begin{tabular}{ccccc}
		\hline
		\multirow{2}{*}{Conditions} & \multicolumn{4}{c}{ $g(w_m) \le$ } \\ \cline{2-5}
		& $10^{-3}$ & $10^{-5}$ & $10^{-10}$ & $10^{-20}$ \\ 
		\hline
		y = 25($\rho$ = 5), D = 10               & 108   & 109   & 118    & 133    \\
		y = 25($\rho$ = 5), D = 20              & 139   & 141   & 149    & 163    \\
		y = 100($\rho$ = 10), D = 10              & 327  & 327  & 340   & 359   \\
		y = 100($\rho$ = 10), D = 20             & 375  & 378  & 387   & 405   \\ 
		\hline
	\end{tabular}
	
\end{table}
	
	We remark here that given finite $y$ and $D$, the output of $g(w_m)$ decreases sharply after $w_m > y+D-2$. Table~\ref{T1. table of k_m} gives the tendency of its decrement. Therefore, despite the value of $\tilde{c}_{D-1} $, we could just simply take a larger $w_m$ to reduce the truncation error $\Delta F_2(\vect{\lambda};\vect{b};y)$ as $\tilde{c}_{D-1}$ is finite. 
	
	However, in contrast to our previous method of fixing the selected degree $w_m$ beforehand, it would be more flexible to adaptively determine that degree online given the different distribution of relative distance between two nodes at each time. As a result, we need to dig into this error equation further to provide some insights on real-time access for the choice of such degree if a desired threshold $\delta f$ of the truncation error is defined, i.e. find a proper $w_m$ so that $\Delta F_2 < \delta f$.

	Firstly, with the help of the recursive computation version of $\tilde{c}_w$ in \eqref{eq. recurrsive equation of c_k+1 and c_k given a equal d_k}, we can derive an explicit equation of $\tilde{c}_w$,
	\begin{equation}
		\begin{aligned}
			\tilde{c}_{k} & = \frac{k-1+d^u}{k}\tilde{c}_{k-1} \\
			&= \frac{\prod_{i=0}^{k-1} (d^u+i) }{\prod_{j=0}^{k-1} (j+1) } \tilde{c}_0. 	
		\end{aligned}
	\end{equation}
	Since $d^u \le D$, $ \prod_{i=0}^{w-1} (d^u+i) \le \prod_{i=0}^{w-1} (D+i) = \frac{(D+w-1)!}{(D-1)!}$, we therefore know that,	
	\begin{equation}
		\tilde{c}_{w} \le \frac{(D+w-1)!}{k!(D-1)!} \tilde{c}_0, \forall w \in \mathbb{N}^+,
	\end{equation}
	and as a result,
	\begin{equation}
		\tilde{c}_{D-1} \le \frac{(2D-2)!}{(D-1)!(D-1)!} \tilde{c}_0,
	\end{equation}	
	then recalling Lemma~\ref{lm. lemma of the factorial of n} yields
	\begin{equation}
		\begin{aligned}
			\tilde{c}_{D-1} &\le \frac{(2D-2)!}{(D-1)!(D-1)!} \tilde{c}_0 \\
			& \le \frac{e(D-1)^{2D-2}}{ \left( \frac{D-1}{e} \right)^{D-1} \left( \frac{D-1}{e} \right)^{D-1} } \tilde{c}_0 \\
			& \le e^{2D-1} \tilde{c}_0 = e^{2D-1} c_0.
		\end{aligned}
	\end{equation}	
	
	Next, let's consider the term $g(w_m)$. If we denote $\alpha = w_m+1-D$, then $g(w_m) = \frac{y^{\alpha+D}}{\alpha!}$. Again, by recalling Lemma~\ref{lm. lemma of the factorial of n}, we get,
	\begin{equation}
		\begin{aligned}
			 g(w_m) = \frac{y^{\alpha+D}}{\alpha!} 
			 \le \frac{y^{\alpha+D}e^{\alpha}}{{\alpha}^{\alpha}} = \left( \frac{ye}{\alpha} \right)^{\alpha} y^D.
		\end{aligned}		
	\end{equation}
	
	So if an error threshold $\delta f$ is given, in order to limit the truncation error under such a value, we can set 
	\begin{equation}
		\Delta F_2(\vect{\lambda};\vect{b};y) \le e^{2D-1}c_0 (1+\frac{\pi^2}{6}) \left( \frac{ye}{\alpha} \right)^{\alpha} y^D < \delta f.
	\end{equation} 
	Since $1+\frac{\pi^2}{6}<e$, the above requirement can be further rewritten as  
	\begin{equation}
		\begin{aligned}
			e^{2D} \left( \frac{ye}{\alpha} \right)^{\alpha} < \frac{\delta f}{c_0 y^D}.
		\end{aligned}
	\end{equation}
	
	Let us denote $\tilde{g}(\alpha) =\left(  \frac{ye}{\alpha} \right)^\alpha$, then similar to the analysis of $g(k_m)$, it's easy to show that $\tilde{g}(\alpha)$ is monotonically decreasing when $\alpha \ge ey, \forall \alpha \in \mathbb{N}^+$. Therefore, we propose two stages to determine the proper degree $\alpha$,
	
	1) if $\alpha \ge e^2y \Rightarrow w_m \ge e^2y + D - 1$, then   
	\begin{equation}
		e^{2D} \left( \frac{ye}{\alpha} \right)^\alpha \le e^{(2D-\alpha)},
	\end{equation}
	
	2) we can let $e^{(2D-\alpha)} < \frac{\delta f}{c_0y^D}$, then
	\begin{equation}
		\begin{aligned}
			& \alpha > 2D - ln \left( \frac{\delta f}{c_0y^D} \right) \\
			\Rightarrow & w_m > 3D - ln \left( \frac{\delta f}{c_0y^D} \right) - 1.
		\end{aligned}
	\end{equation}
	
	Considering the conditions both in the Basel Problem \eqref{eq. condition of Basel Problem} and in Corollary~\ref{Corollary of d_k}, we can summarize the set of sufficient conditions as listed in \eqref{eq. sufficient and necessary conditions of in bounded theorem of k_m}.
	
	This ends the proof of Theorem \ref{Th. boundedness of truncation error}.
\end{proof}


\subsection{Translational Approximate Covariance Expansion} \label{subsec. TRACE}

Theorem~\ref{Th. boundedness of truncation error} demonstrates that one can reduce the truncation error introduced by the finite-term approximation in \eqref{eq. PSE cumulative probability function, finite terms } by increasing the selected degree $w_m$. The rules of determining $w_m$ show its relationship between the shape of the covariance, the range limitation and also the desired error threshold. It presents a guaranteed and numerically efficient computation method for the probability of connection. However, all aforementioned analyses are built upon a very important assumption in \eqref{eq. sufficient and necessary conditions of in bounded theorem of lamda}, that is, the proposed finite-time approximation can only be efficient when the given relative covariance is well-constructed i.e. its eigenvalues are greater than $\frac{1}{2}$. 

This is a very critical constraint since in real applications the robot beliefs continue to be updated after each new control action or whenever a new measurement is obtained. Demonstrated as the second disadvantage in our previous work \cite{zhang2020connectivity}, this constraint will cause fatal prediction failure when using the finite-term approximation in \eqref{eq. PSE cumulative probability function, finite terms } for the case where the relative covariance is too small or near singularity and the relative distance of the two corresponding robots $d_{12} = \left\lVert\vect{p}_1 - \vect{p}_2\right\rVert_2$ is very close to the range threshold $\rho$.   

Though this constraint is partly relieved by the Approximate Covariance Expansion (ACE) proposed in our prior research, ACE suffers from lots of approximations and hence the probability precision is corrupted and degenerated. Regarding these problems, we further propose an improved version of ACE, where the expansion is transitioned to a new and collinear center instead of dilating the covariance at its ellipse center. This approach is hereafter called the Translational Approximate Covariance Expansion (TRACE). A geometric depiction of TRACE is shown in Fig. \ref{fig. Trace}, where the red solid ellipse $E_{r1}$ represents the $3\delta$ confidence area corresponding to the covariance matrix $\Sigma = \Sigma_1 +\Sigma_2$ of $\Delta \bar{P}$. Taking the maximum range of the measurement device to be $\rho_{1}$, its ranging circle from the origin $O$ with radius $\rho_{1}$ intercepts with the $3\delta$ area of $E_{r1}$. There exists a smaller ellipse $E_{b1}$ that is tangent to this ranging circle. Therefore the probability $p(Y<\rho^2_1)$ can be indicated by 
\begin{equation}
	p(Y<\rho^2_1) =  \frac{S(\widearc{ABC} \cap E_{r1})}{S(E_{r1})}, 
\end{equation} 
where $S(\cdot)$ represents the shape's area of input arguments. Now suppose that the covariance $\Sigma$ is very small or near singular, then we enlarge the covariance $\Sigma$ by a coefficient $\beta \ge 1$ and translate it into the solid red ellipse $E_{r2}$.

\begin{figure}[htbp] 
	\centering
	\includegraphics[width=0.7\linewidth]{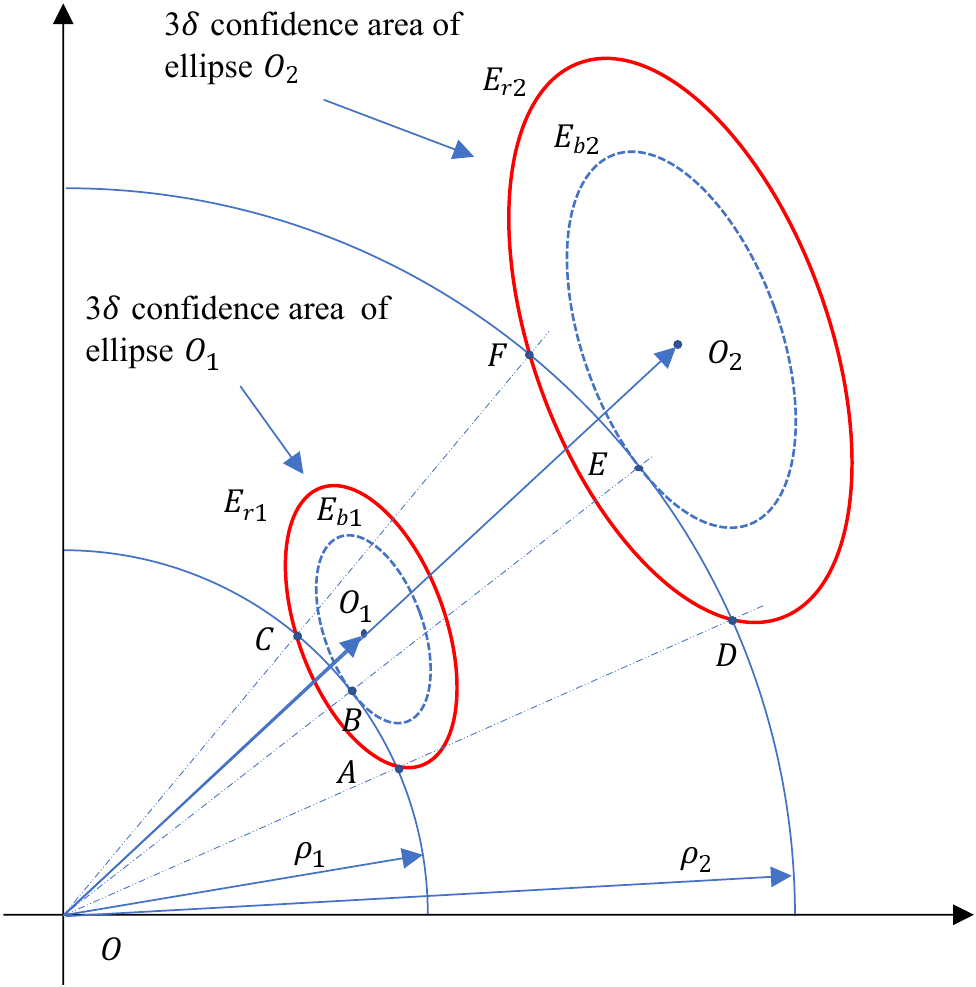}
	\caption{Geometric illustration of TRACE.}
	\label{fig. Trace}
	\vspace{-10pt}
\end{figure}

\begin{algorithm}[htbp]
	\SetAlgoLined
	\KwResult{ $\vect{\mu}_2,\vect{\Sigma}_2,\rho_2$ \; }
	\KwData{ $\vect{\mu}_1,\vect{\Sigma}_1,\rho_1$ \;}
	
	$\lambda_{min} \leftarrow$ compute the minimum eigenvalue of $\vect{\Sigma}_1$  \;
	
	\eIf{$\lambda_{min} < \frac{1}{2}$} {
		$\beta = \frac{1}{\lambda_{min}},$
		$\vect{\Sigma}_2 = \beta \vect{\Sigma}_1,$
		$\vect{\mu}_2 = \sqrt{\beta} \vect{\mu}_1,$ 
		$\rho_2 = \sqrt{\beta}  \rho_{1}$ \;
	}
	{		$\vect{\Sigma}_2 =  \vect{\Sigma}_1,$
		$\vect{\mu}_2 =  \vect{\mu}_1$ 
		$\rho_2 =  \rho_{1}$ \;
	}
	\textbf{Return} $\vect{\mu}_2,\vect{\Sigma}_2,\rho_2$
	\caption{TRACE}
	\label{alg. TRACE}
\end{algorithm}

In contrast to ACE, the expanded ellipses are collinear with the original ellipses, i.e. $E_{r2}$ and $E_{b2}$ in TRACE are moved to a new center $O_2$, which is collinear with the original center $O_1$ and the coordinate origin $O$. In doing so, we can simply derive the following equation according to basic geometric similarity,
\begin{equation}
	\frac{S(\widearc{ABC} \cap E_{r1})}{S(E_{r1})} =  \frac{S(\widearc{DEF} \cap E_{r2})}{S(E_{r2})},
\end{equation} 
if $\rho_2 = \sqrt{\beta} \rho_1$. 

\begin{algorithm}[tbp]
	\SetAlgoLined
	\KwResult{ \\
		$ p(Y<\rho^2)$: future connection probability $p(s_{12}^{k+l}=1)$ \;} 
	\KwData{ \\ 		
		$ b(\vect{X}^{k+1:k+L|k}) \sim N(\vect{\bar{X}}^{k+1:k+L|k},\vect{\Sigma})$: propagated belief \; 
		$\rho$ : ranging maximum of sensors\; 
		$\delta f$ : desired theoretical limit for the truncation error. }
	
	Extract $\vect{p}_1^{k+l} \sim N(\vect{\bar{p}}_1^{k+l|k}, \vect{\Sigma}_1),\vect{p}_2^{k+l} \sim N(\vect{\bar{p}}_2^{k+l|k},\vect{\Sigma}_2),$ from $b(\vect{X}^{k+1:k+L|k})$:
	
	Compute $\Delta \vect{p}^{k+1}_k\sim N(\vect{\mu},\vect{\Sigma})$ by: \\
	
	 $\qquad \vect{\mu} = \vect{\bar{p}}_1^{k+1|k} - \vect{\bar{p}}_2^{k+1|k}, \vect{\Sigma} = \vect{\Sigma}_1 + \vect{\Sigma}_2 -2\vect{\Sigma}_{12}$\\
	
	$[\rho_{3\delta}^{min} \rho_{3\delta}^{max}] \leftarrow 3\delta$ area identification \cite{zhang2020connectivity} \;
	
	\uIf{$\rho < \rho_{3\delta}^{min}$}{
		\textbf{Return} $p(Y<\rho^2) = 0$ \;
	} 
	\ElseIf{$\rho > \rho_{3\delta}^{max}$}{
		\textbf{Return} $p(Y<\rho^2) = 1$ \;
	} 
	$\left[ \vect{\mu},\vect{\Sigma},\rho \right] \leftarrow \text{TRACE}(\vect{\mu},\vect{\Sigma},\rho)$ \tcc*{algorithm~\ref{alg. TRACE}}
	
	$\vect{\lambda} \leftarrow$ compute the eigenvalues of $\vect{\Sigma}$\;
	
	$\vect{P} \leftarrow$ compute the orthogonal matrix composed of the eigenvectors of $\vect{\Sigma}$ \;
	$\vect{b} = P^T\vect{\Sigma}^{-\frac{1}{2}}\vect{\mu}$,	$y = \rho^2 $, $w_m = 50$\;
	
	\For{$r=0;r\le w_m;r=r+1$}{
		$d_r \leftarrow $ using \eqref{eq. computation of the coefficient d_k} for $r = 1,...,w_m$\;
		$c_r \leftarrow $ using \eqref{eq. computation of the coefficient c_k} by using $d_r$ for $r = 0,...,w_m$\;
		$y_r = (-1)^r \dfrac{y^{r+1}}{\Gamma(r+2)}, r = 0,...,k_m $  \tcc*{$y_r$ is the right side of \eqref{eq. PSE cumulative probability function, finite terms } without $c_r$}
		\If(\tcc*[h]{adaptive strategy for selecting $w_m$}){$d_r$ reaches its maximum}{
			$D = d_r$,  $w_{m1} = \frac{1}{b_1^2} + \frac{1}{2\lambda_1 - 1}$ ,	$w_{m2} = \frac{1}{b_2^2} + \frac{1}{2\lambda_2 - 1}$, $w_{m3} = e^2y + D$ , $w_{m4} = 3D - ln \left( \frac{\delta f}{c_0y^D} \right) - 1$ \;
			
			$w_m = \text{max}(w_{m1},w_{m2},w_{m3},w_{m4},w_m)$ \;
		}
	}
	$ F_p(\vect{\lambda};\vect{b};y) \leftarrow $  \eqref{eq. PSE cumulative probability function, finite terms } using the series $c_r, y_r$ \;
	
	\textbf{Return} 	$p(Y<\rho^2) = F_p(\vect{\lambda};\vect{b};\rho)$
	\caption{APSE algorithm for CPP}
	\label{alg. probability-prediction algorithm}
\end{algorithm}

\blue{The pseudocode of TRACE is presented in Algorithm~\ref{alg. TRACE}.} It should be noted that there are various ways to choose a proper expansion coefficient $\beta$. In the following algorithm, we calculate $\beta$ by expanding the covariance matrix to a new one whose smallest eigenvalue is no less than 1. Although TRACE is proposed to deal with the case of a singular or small covariance, we could also revise it inversely by narrowing a larger ellipse at $O_2$ to a smaller but `well-constructed' one at $O_1$. This way we can avoid a very large ranging threshold $\rho$ in computation because $y=\rho^2$ is heavily used in our finite-term approximation in \eqref{eq. PSE cumulative probability function, finite terms }. A large $\rho$ may also introduce and accumulate rounding errors from numerical computation on a hardware platform.  

\subsection{Adaptive Power Series Expansion (APSE)}

So far, we have presented the two main parts of our algorithm. In Subsection~\ref{subsec. Stability Analysis} we provided a theoretical derivation showing that Theorem~\ref{Theorem: Power series expansion of quadratic form of normal distribution} can accurately calculate the probability of connection of two nodes as long as the covariance of the relative distance is proper and most importantly that the degree of the finite-term approximation can be adaptively selected. Then in Subsection~\ref{subsec. TRACE} we provided a novel method for dealing with the constraint of the covariance to extend our method to more general cases. Jointly, we conclude the pseudocode for computing the probability of a future connection in Algorithm~\ref{alg. probability-prediction algorithm}.

As shown by the top-down order, this algorithm takes as inputs the estimated mean $\vect{\mu}$, covariance $\vect{\Sigma}$ of the relative distance distribution of two nodes and the ranging threshold $\rho$. It finally returns the probability of the event $Y = d_{12}^2 < \rho^2$.  Lines $4-8$ are the content of $3\delta$ regulation, which has been developed and explained in our previous work \cite{zhang2020connectivity}. The purpose of preserving this part is to reduce the computational complexity when the relative distance computed from the estimated means lies far beyond the $99.7\%$ sample coverage area (i.e. $3\delta$ area). So we just simply assign the probability to be 1 or 0 and accept the consequently possible $0.03\%$ error.

Moving on, the TRACE algorithm is called in line $10$ to derive a proper covariance $\vect{\Sigma}$ whose eigenvalues meet the critical constraint in \eqref{eq. sufficient and necessary conditions of in bounded theorem of lamda}. Meanwhile, the mean vector $\vect{\mu}$ and ranging threshold $\rho$ would be updated accordingly. Then lines $11-22$ are the computation of our proposed finite-term approximation of Theorem~\ref{Theorem: Power series expansion of quadratic form of normal distribution}. Within the pseudocode, we firstly give an initial value of the selected degree $w_m = 50$ so that we can find the maximum of the series $d_w$. Note that the initialization of $w_m$ presented here is not necessary but is instrumental in the pre-allocation of storage. Here, we suppose that $d_w$ can reach its peak within the first 50 elements and therefore our code could enter lines $18-21$. Here are the main results of Theorem~\ref{Th. boundedness of truncation error} for adaptively choosing a proper degree $w_m$ for the summation. We note that this part only needs to be run once and since we compare all the candidates of $w_m$ with the initial at line $20$ the total number of summations of our finite-term approximation in \eqref{eq. PSE cumulative probability function, finite terms } will be at least 50. 

\section{Simulations and results}

In this section, we present extensive analyses of the proposed methods and perform comparisons against benchmark methods from related works. 
We first evaluate the performance of APSE in Algorithm~\ref{alg. probability-prediction algorithm} in terms of accuracy and computational efficiency. \blue{Then, two instances of the CLAP problem are investigated. The first is drawn from our prior work \cite{zhang2020connectivity} and formulates the problem as a one-step finite state MDP (OS-MDP) where the state space consists of all possible network topologies and the reward is a function of the leader's localization uncertainty. The second is derived from a patrolling task and we study the impact of a longer planning horizon by using a generalized belief space formulation (GBS).} 

In addition, to show the effectiveness of the proposed APSE algorithm in active planning problems, we compare against several prevailing approaches that solve the CPP in existing works, these are

1) a deterministic \emph{Bernoulli} model in Algorithm~\ref{alg. Bernoulli distribution for disk communication model}. 

2) a \emph{Linear} distribution in Algorithm~\ref{alg. linear distribution for disk communication model}. 

3) a \emph{Random} sampling-based distribution in Algorithm~\ref{alg. RandomSampleDistributionAlg}.

\begin{algorithm}[htbp]
	\SetAlgoLined
	\KwResult{ $p(Y<\rho^2)$  }
	\KwData{ $\vect{\mu}$, $\rho$ \;}
	
	$Pr(Y<\rho^2) = 0$ \tcc*{Initialized with 0} 
	\If{ $||\vect{\mu}||_2 \le$ $\rho$}{
		$p(Y<\rho^2) = 1 $
	}
	
	\textbf{Return} $p(Y<\rho^2)$
	\caption{Bernoulli Model}
	\label{alg. Bernoulli distribution for disk communication model}
\end{algorithm}

\begin{algorithm}[htbp]
	\SetAlgoLined
	\KwResult{ $p(Y<\rho^2)$  }
	\KwData{ $\vect{\mu}$, $\vect{\Sigma}$, $\rho$ \;}
	
	$[\rho_{3\delta}^{min}, \rho_{3\delta}^{max} ] \leftarrow$ $3\delta$ area identification; \\ 
	\uIf{$\rho \le \rho_{3\delta}^{min}$ }{
		\textbf{Return} $Pr(Y<\rho^2) = 1$
	}
	\ElseIf{$\rho > \rho_{3\delta}^{min}$ }{
		\textbf{Return} $Pr(Y<\rho^2) = 0$
	}
	
	\textbf{Return} $p(Y<\rho^2) = \frac{1}{\rho_{3\delta}^{max}-\rho_{3\delta}^{min}}(\rho - \rho_{3\delta}^{min})$ ; \\
	\caption{Linear model}
	\label{alg. linear distribution for disk communication model}
\end{algorithm}

\begin{algorithm}[htbp]
	\SetAlgoLined
	\KwResult{ $p(Y<\rho^2)$  }
	\KwData{ $\vect{\mu},\vect{\Sigma},\rho, dim$ \;}
	
	Count = 0 \;
	\While{$k \le dim$}{
		data $\leftarrow$ RandomlySample($\vect{\mu},\vect{\Sigma}$) \;
		\If{ $||data||_2$ $ \le \rho$}
		{Count = Count + 1 \;}
	}
	\textbf{Return} $p(Y<\rho^2) = \frac{Count}{dim}$\;
	\caption{Random sampling-based model}
	\label{alg. RandomSampleDistributionAlg}
\end{algorithm}

Note that the Bernoulli distribution has been widely applied to many existing active planing problems \blue{by taking the well known ML assumption \cite{platt2010belief} }.  This is due to the similarity it offers to the estimation problem, whereby the standard estimation engines (for both the filtering method and smoothing formulation) can be directly introduced to recover the future beliefs given a control candidate. As a result the corresponding cost can be easily derived. The linear distribution, however, is a partly empirical formula which does not consider the exact motion and measurement models. \cite{indelman2015planning} came up with this method in their problem formulation to show the influence of uncertain observations but no further detail is given about how the active planning framework can deal with such a probabilistic future measurement and communication topology. Lastly, the random sample distribution is a completely intuitive method which uses the frequency of a connection to represent its probability \blue{as used in \cite{pathak2018unified} }. However, the number of samples may lead to different performance in terms of predictive precision and computational tractability. 


The code for all simulations was implemented in Matlab R2019a and executed on a Windows machine with an i7-9750H @ 2.6 GHz processor and 8 GB of memory.

\begin{figure}[htbp] 
	\centering
	\includegraphics[width=0.8\linewidth]{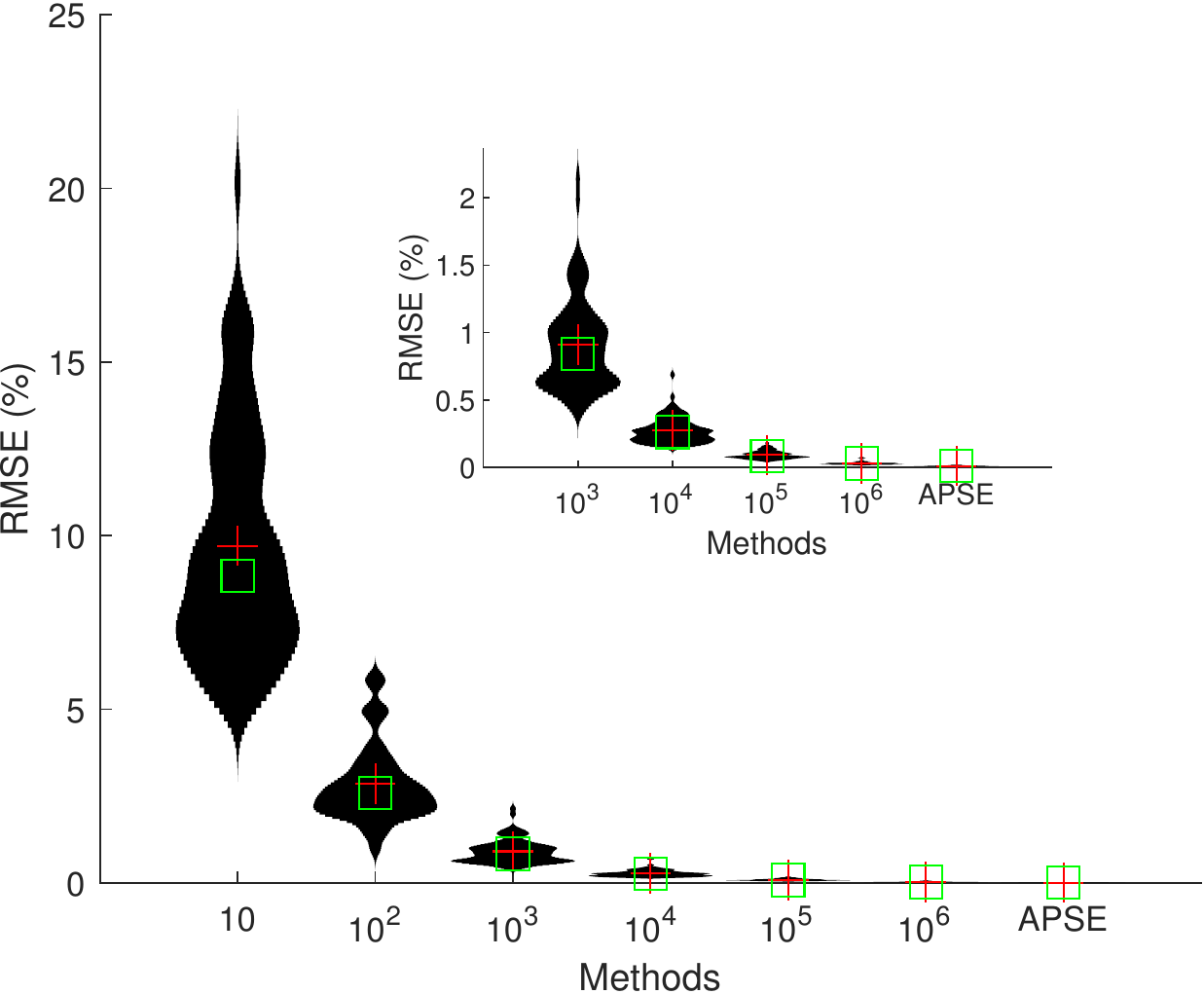}
	\caption{RMSE distribution over 200 trials.}
	\vspace{-15pt}
	\label{fig. all case RMSE}
\end{figure}
\begin{figure}[htbp] 
	\centering
	\includegraphics[width=0.8\linewidth]{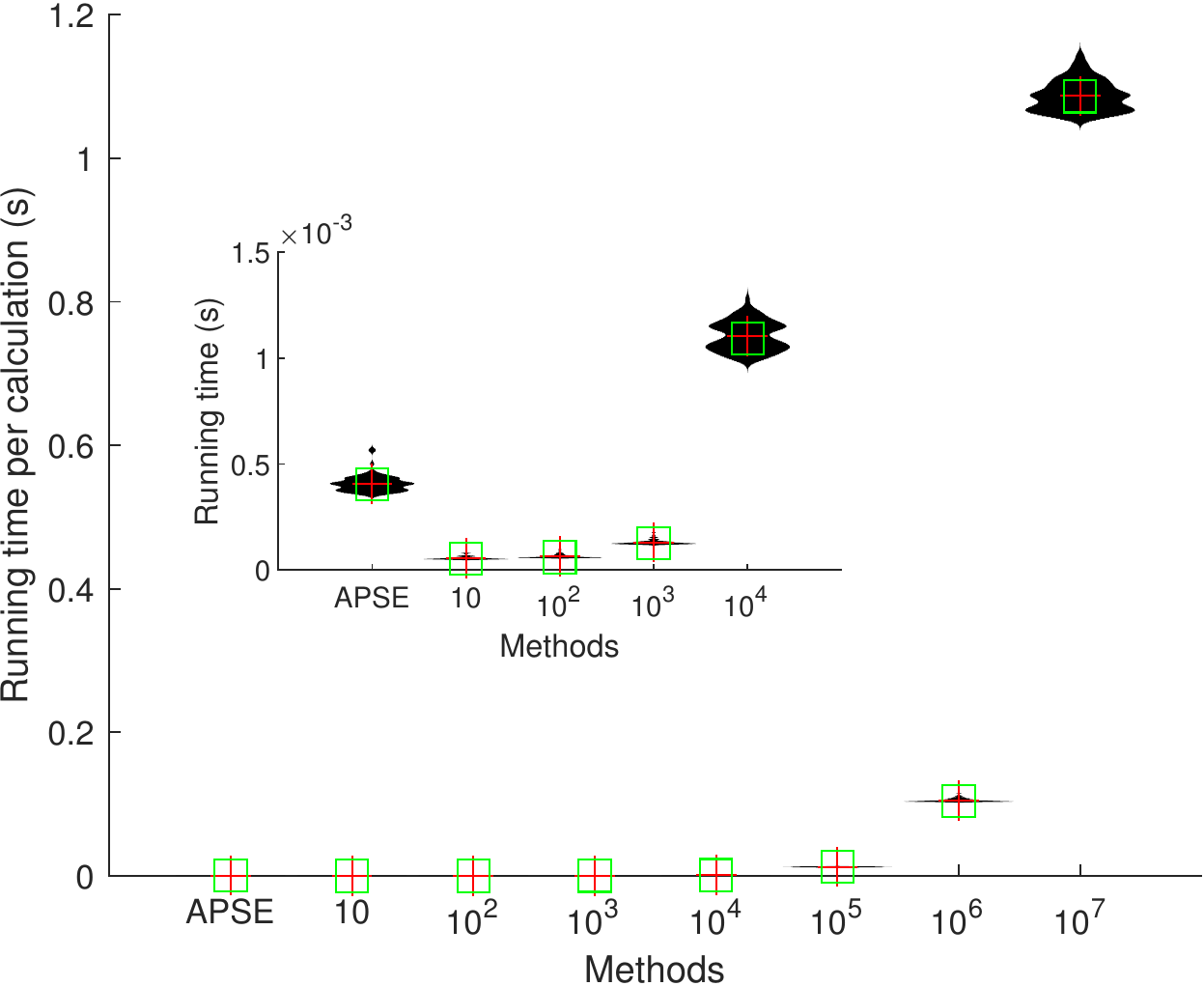}
	\caption{Mean run time distribution over 200 trials.}
	\vspace{-15pt}
	\label{fig. all case running time}
\end{figure}

\subsection{Performance of APSE}
In this subsection, we test our proposed APSE in Algorithm~\ref{alg. probability-prediction algorithm} to evaluate its probability accuracy and computational complexity. By comparing against existing methods, we will show that one great advantage of APSE is its high precision. Since the Bernoulli and Linear models are merely two rough guesses of the real distribution of the connection, we only compare our proposed algorithm against the random sampling-based method with different sampling degrees.

Before moving on, we remark here that the implementation of our proposed probability algorithm is realized using the for-loop in MATLAB, while the comparison methods using random sampling apply the built-in function \texttt{mvrnd}. Therefore, the computation time performance of our proposed APSE algorithm, presented in this subsection, could be further improved if more optimization can be introduced.

The setup is given as follows. We test two nodes $A$ and $B$ whose estimated means are fixed to $\vect{p}_A = [0.5,1]^T$ and $\vect{p}_B = [2,2.5]^T$ and the mean distance between them is $||\vect{p}_A-\vect{p}_B||_2 \approx 2.21$. In each simulation trial, we do the following actions:

1) Randomly generate two symmetric and positive matrices $\vect{P}_A>0$, $\vect{P}_B>0$ as the covariance of nodes $A$ and $B$. 

2) Experimental method: for each range $\rho$ in the sets $\vect{S}_{\rho} := \{ 0.1:0.1:6 \}$,   $\vect{p}_A,\vect{p}_B,\vect{P}_A,\vect{P}_B$ and $\rho$ are taken as the input of APSE with the desired accuracy $\delta f = 10^{-10}$. We record the output probability and run time corresponding to each $\rho \in \vect{S}_{\rho}$ into the set $\vect{P}_{APSE}$ and $\vect{T}_{APSE}$, whose sizes are both 60.

3) Comparison methods: Random sampling (Algorithm~\ref{alg. RandomSampleDistributionAlg}) is tested for six different sampling degrees i.e. $dim = [10, 10^2, 10^3,10^4,10^5,10^6,10^7]$. For each $dim$, repeat step 2) and record six sets of probabilities $\vect{P}_{dim}$ and run times $\vect{T}_{dim}$. When $dim = 10^7$, the probability is treated as the ground truth i.e. $\vect{P}_{dim=10^7} = \vect{P}_{truth}$.

4) Compute the root mean squared error (RMSE) for each method in $\{ APSE, 10, 10^2, 10^3,10^4,10^5,10^6,10^7 \}$:
$$RMSE_{method} = \sqrt{\sum_{i=1}^{60}\left[\vect{P}_{method}(i) - \vect{P}_{truth}(i)\right]^2},$$
and compute the mean run time per calculation as:
$$MeanRunTime = \frac{1}{60} \sum_{i=1}^{60} (\vect{T}_{method}(i)).$$

The statistical distribution of both RMSE and mean run time per calculation over 200 trials are presented in Fig.~\ref{fig. all case RMSE} and Fig.~\ref{fig. all case running time}, respectively. It is not surprising to see that the comparison group, the random sampling-based algorithm, produces higher calculation accuracy but worse computational complexity as the sampling degree gets larger. The experimental group, the APSE Algorithm~\ref{alg. probability-prediction algorithm}, surpasses all the comparison groups in terms of probability precision, which is indicated by RMSE. At the same time, it achieves a relatively moderate computational complexity, remaining in the same order of magnitude as the random sampling algorithm with degree between $10^3$ and $10^4$.

\color{blue}
\subsection{CLAP with OS-MDP}

As depicted in Fig.~\ref{fig. init}, this simulation follows a similar setup to our previous work \cite{zhang2020connectivity}. Here we consider a MRS consisting of one leader and four followers in a GNSS-limited and noisy environment. All robots only know their initial positions exactly (as shown in Table~\ref{tab. initialPosSimulation1}) and need to localize themselves due to the existence of motion noise. A point-mass motion model corrupted with Gaussian noise is considered:
$$\vect{p}_i^{k+1} = \vect{p}_i^{k} + \vect{u}_i^{k} + \vect{w}_i,$$
where $\vect{w}_i \sim \mathcal{N}(\vect{0},\vect{R}),\forall i \in \mathcal{V}$ with known information matrix $\vect{R} = diag([0.04, 0.04])$. The leader is initially located at the bottom-left corner and is tasked to reach its destination by traversing the mission plane along a predefined trajectory. The leader's motion strategy used in this simulation is calculated simply by heading toward the destination i.e.
$$\vect{u}_L = u_{max} \frac{\vect{p}_{des} - \vect{\hat{p}}}{||\vect{p}_{des} - \vect{\hat{p}}||_2},$$
where $\vect{p}_{des}$ is the position of the destination, $\vect{\hat{p}}$ is the estimated position of the leader and $u_{max}$ is the limited moving distance at each time step. 

\begin{figure}[htbp]
	\centering
	\includegraphics[width=0.48\textwidth]{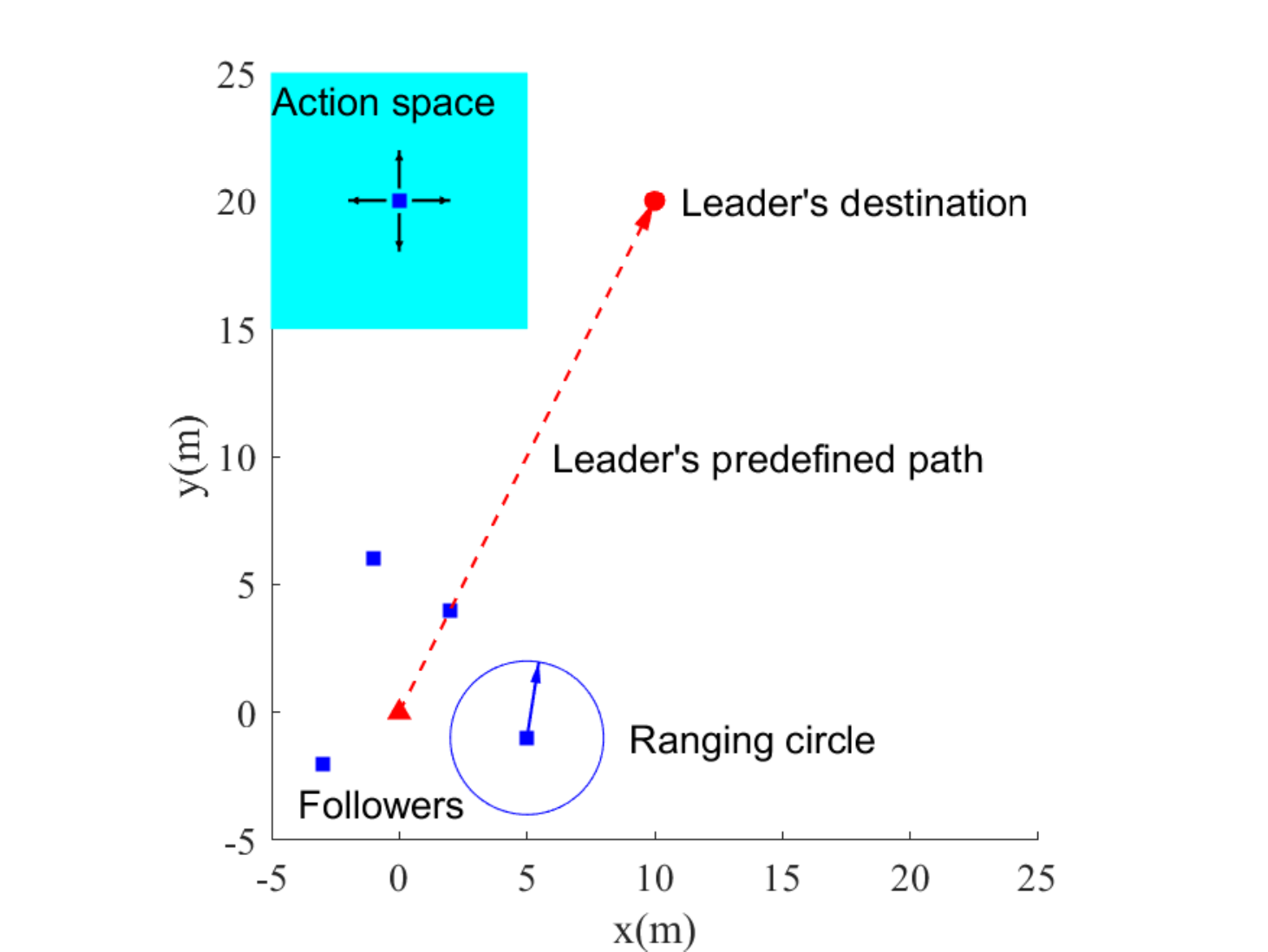}
	\caption{\textcolor{blue}{Initial configuration of the CLAP problem with an OS-MDP active planning framework.}}
	\label{fig. init}
	\vspace{-10pt}
\end{figure}

\begin{table}[htbp]
    \centering
    \caption{\blue{Initial positions for the leader and four followers.}}
    \label{tab. initialPosSimulation1}
    \begin{tabular}{cccccc}
    \hline
    Nodes  & L         & F1         & F2        & F3         & F4          \\
    Pos(m) & {[}0;0{]} & {[}-1;6{]} & {[}2;4{]} & {[}5;-1{]} & {[}-3;-2{]} \\ \hline
    \end{tabular}
\end{table}

Four followers are deployed to help reduce the leader's accumulated localization uncertainty during the mission by optimizing their motion sequences. The action space for each follower is set as,
$$\vect{A}_i = {[0;1], [1;0], [0; -1], [-1;0]}.$$ 
All robots equip ranging devices whose maximum measurement distance is set $\rho = 3 m$. As no global localization is available in the mission space and no landmarks or anchors are available, the relative measurements between robots are the only information source for robots to correct their localization uncertainty. The measurement model used in the simulation is exactly \eqref{eq. observation model for robots_realtiveDistance} with known information matrix $Q = 0.01$ and the active planning framework is the OS-MDP \cite{zhang2020connectivity} with a one-step-ahead planning horizon. The immediate cost is computed as the leader's localization uncertainty,
$$c^l\left( b(\vect{x}^{k+1,\vect{U}^k})  \right) = trace \left(  \vect{\Sigma}_{\vect{p}^{k+1}_L} \right).$$

All four methods of predicting the probability of future connectivity are tested for 50 trials and the statistical results are shown in Figs~\ref{fig. StatisticalResultsCLAPwithOS-MDP}-\ref{fig. MeanUncertaintiesOneFigure-CLAP-OS-MDP}. Besides, the mean and standard deviation of the trace of leader's localization covariance at the final step are also calculated and listed in Table~\ref{tab. mean-std-improvement-rate}.

\begin{table}[htbp]
\centering
\caption{\textcolor{blue}{Final trace of the leader's localization covariance.}}
\label{tab. mean-std-improvement-rate}
\begin{tabular}{ccc}
\hline
Methods   & \multicolumn{1}{c}{Mean $\pm$ std (m)}   & APSE improvement  \\ \hline
Bernoulli & 0.4539 $\pm$ 0.1914                & 20.8\%     \\
Linear    & 0.4114 $\pm$ 0.1774             & 12.6\%          \\
Random    & 0.3958 $\pm$ 0.0818               & 9.2 \%        \\
APSE      & \textbf{0.3595 $\pm$ 0.0719 }         & --             \\ \hline
\end{tabular}
\end{table}

\begin{figure}
	\centering
	\begin{subfigure}[b]{0.24\textwidth}
	    \centering
		\begin{subfigure}[b]{1\textwidth}
			\centering
			\includegraphics[width=\textwidth]{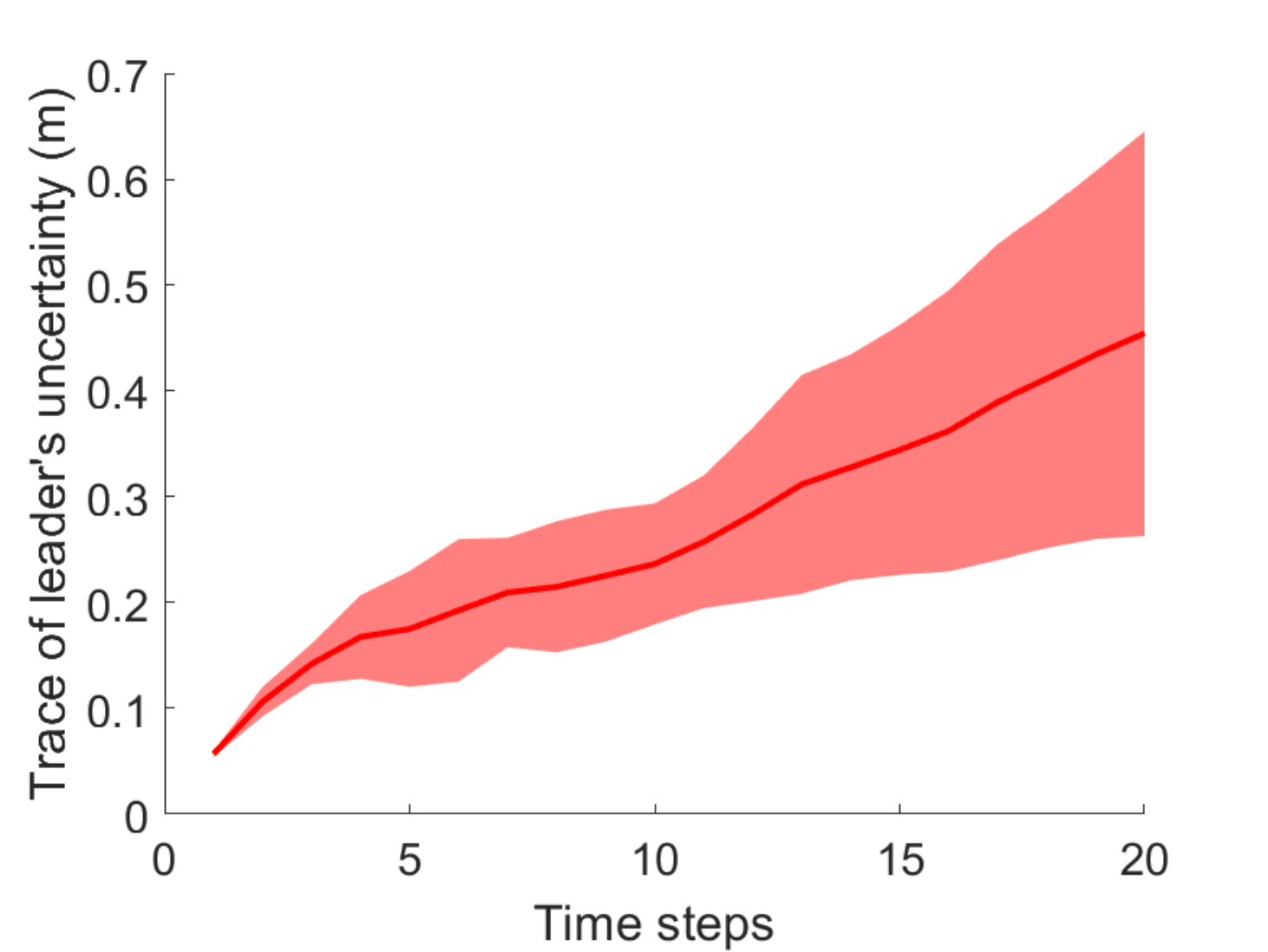}
			\caption{Bernoulli}
		\end{subfigure}
		\null \hfill \\
		\begin{subfigure}[b]{1\textwidth}
			\centering
			\includegraphics[width=\textwidth]{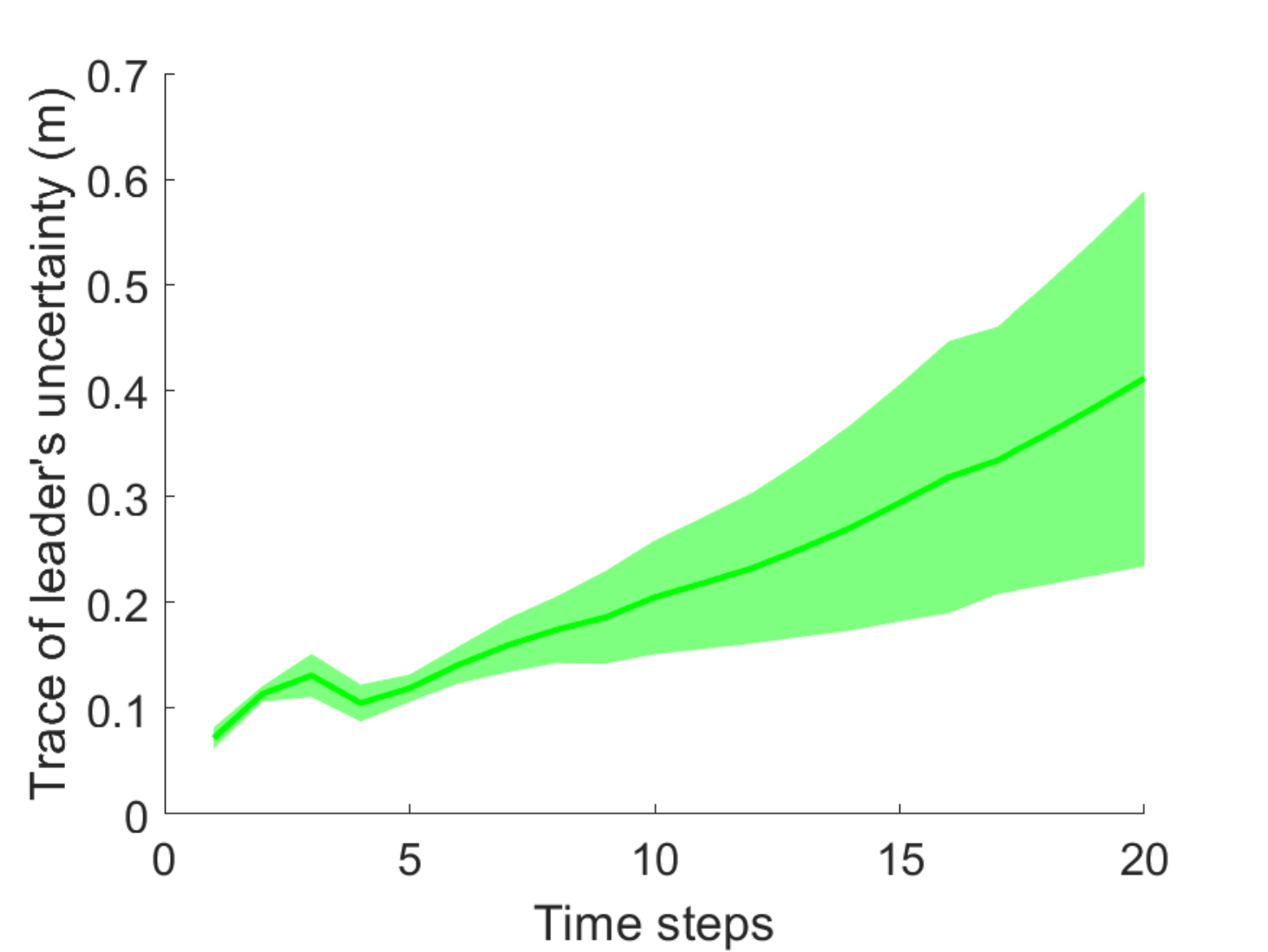}
			\caption{Linear}
		\end{subfigure}
		
	\end{subfigure}
	\hfill
	\begin{subfigure}[b]{0.24\textwidth}
	    \centering
		\begin{subfigure}[b]{1\textwidth}
			\centering
			\includegraphics[width=\textwidth]{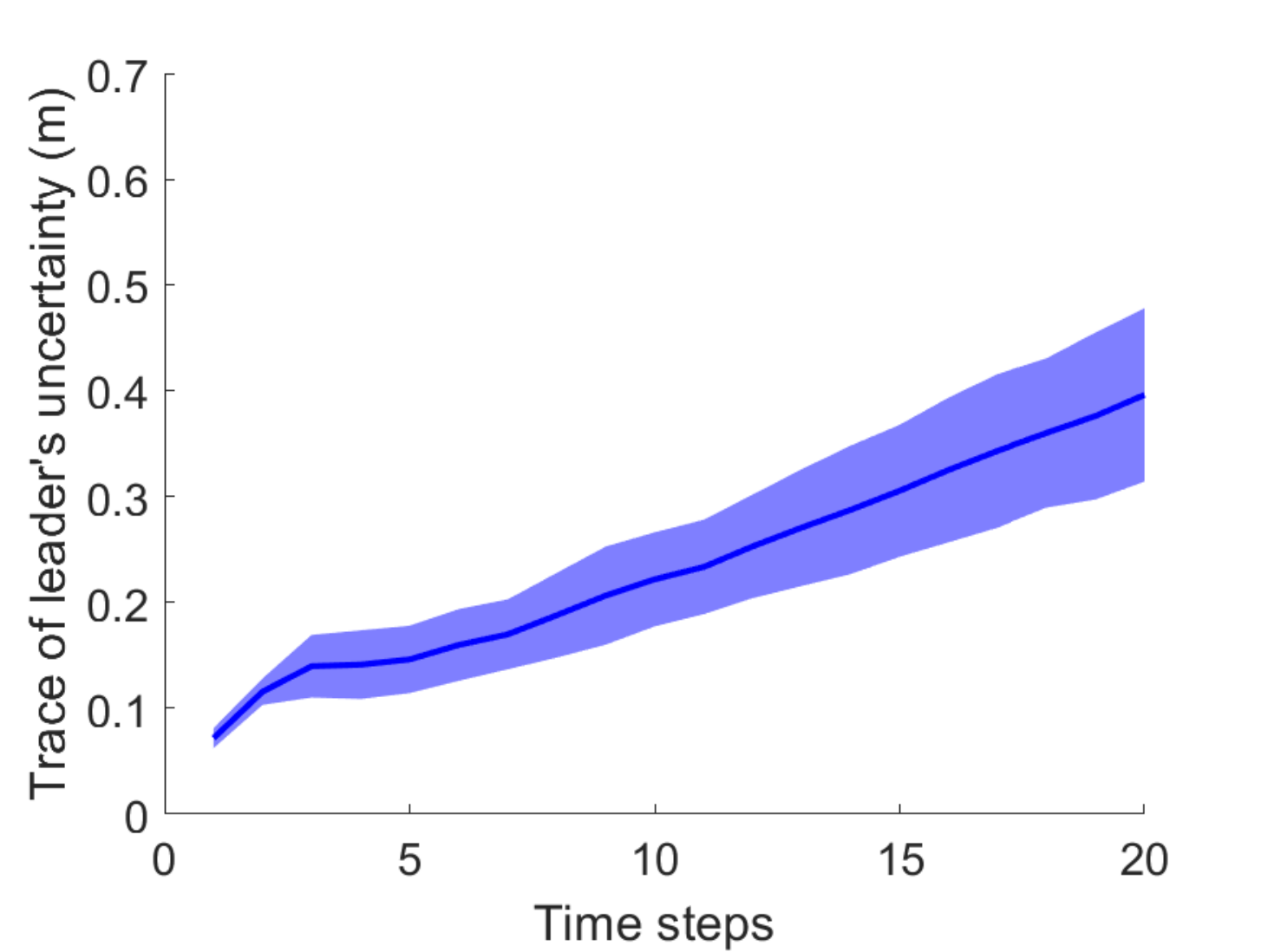}
			\caption{Random with $dim=10^4$}
		\end{subfigure}
		\null \hfill \\
		\begin{subfigure}[b]{1\textwidth}
			\centering
			\includegraphics[width=\textwidth]{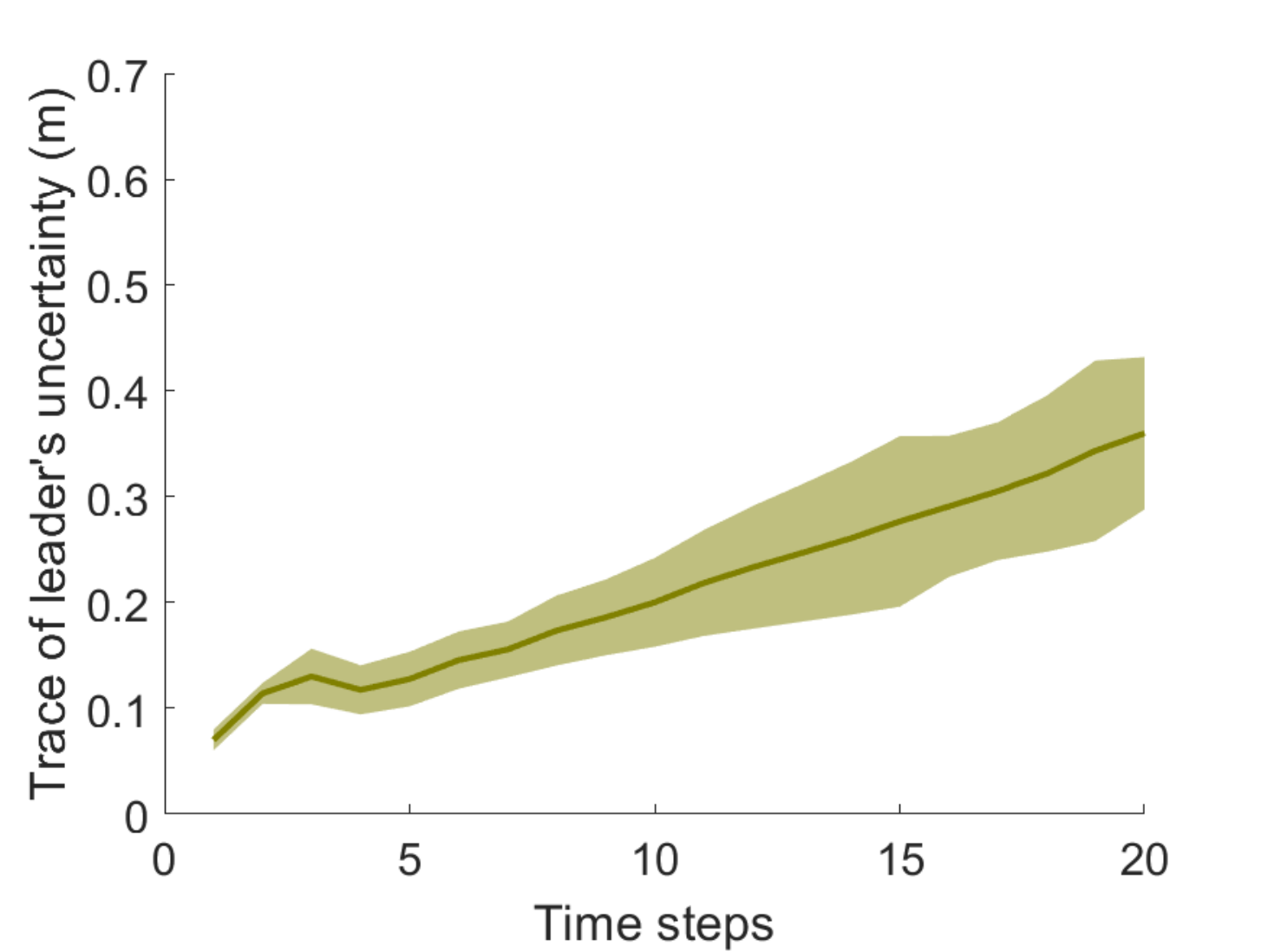}
			\caption{APSE}
		\end{subfigure}
	\end{subfigure}
	\caption{ \blue{(a)-(d) Statistical distribution of the leader's localization uncertainty over 50 trials for CLAP with OS-MDP.}}
	\label{fig. StatisticalResultsCLAPwithOS-MDP}
	\vspace{-5pt}
\end{figure}

\begin{figure}[htbp]
	\centering
	\includegraphics[width=0.45\textwidth]{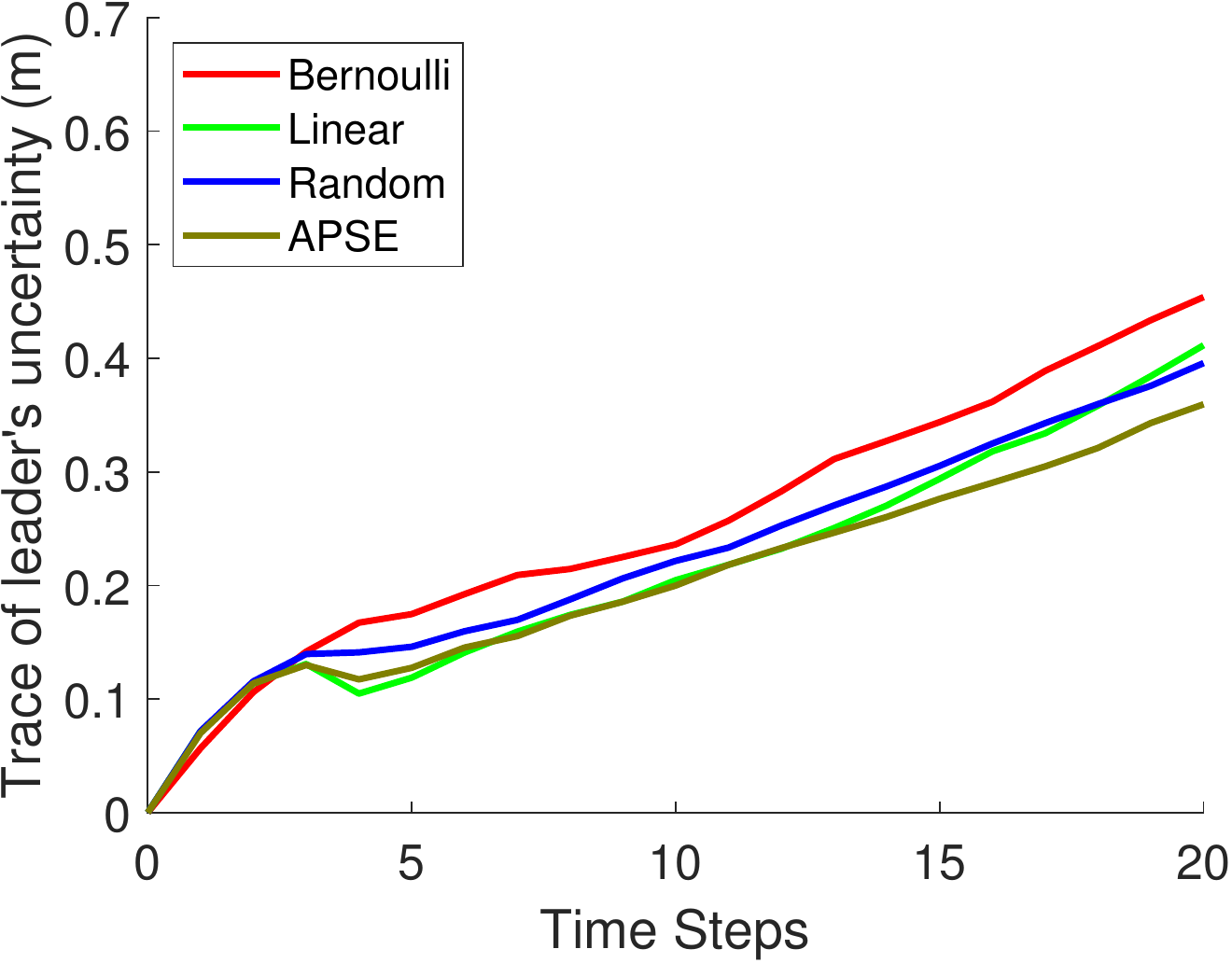}
	\caption{\textcolor{blue}{Leader's mean uncertainties over 50 trials for all four methods.}}
	\label{fig. MeanUncertaintiesOneFigure-CLAP-OS-MDP}
	\vspace{-15pt}
\end{figure}

It is not surprising to see that the distributions of the trace for all four methods all keep increasing and diverging. There is no fixed location source and therefore the absolute drift of the overall system cannot be corrected. More theoretical results of this observation can be found in \cite{mourikis2006performance}. However, we also observe that our proposed APSE algorithm surpasses all the existing methods i.e. \{Bernoulli, Linear and Random\}. Our proposed APSE algorithm has the slowest localization performance decay and improves over the other methods by \{20.8\%, 12.6\%, 9.2\%\}, respectively. Therefore, we conclude that a more accurate probability of future connection can indeed help improve the performance of active planning in noisy environments.  

\color{black}

\subsection{CLAP \blue{with generalized belief space (GBS)}}

\blue{We investigate a more realistic MRS planning scenario based on a patrolling mission and allow a longer planning horizon.} Here, two mobile robots operate in a noisy and GNSS-limited environment, of which robot $L$ is the leader robot responsible for the high-level tasks that are assigned by a human operator. The follower robot $F$ is scheduled to help the leader reduce its localization uncertainty. \blue{The follower serves} as a relay node that actively and optimally generates extra communication and observation links when the leader is out of the anchor's coverage area. \blue{The single anchor in the scene knows its exact position and can communicate with and observe nodes within its coverage area.} The initial configuration is illustrated in Fig.~\ref{fig. CLAPconfigureation}, where the leader is executing a cruise or monitoring mission whose predefined trajectory is completely out of the communication and observation range of the anchor. Through a relay node, the leader is virtually connected to the anchor by the relay link as demonstrated by the green dashed lines in Fig.~\ref{fig. CLAPconfigureation}. The quality of such a virtual connection depends on the relative position of the measurement graph, which can be controlled by optimizing the follower's trajectory. 

\begin{figure}[htbp]
	\centering
	\includegraphics[width=0.45\textwidth]{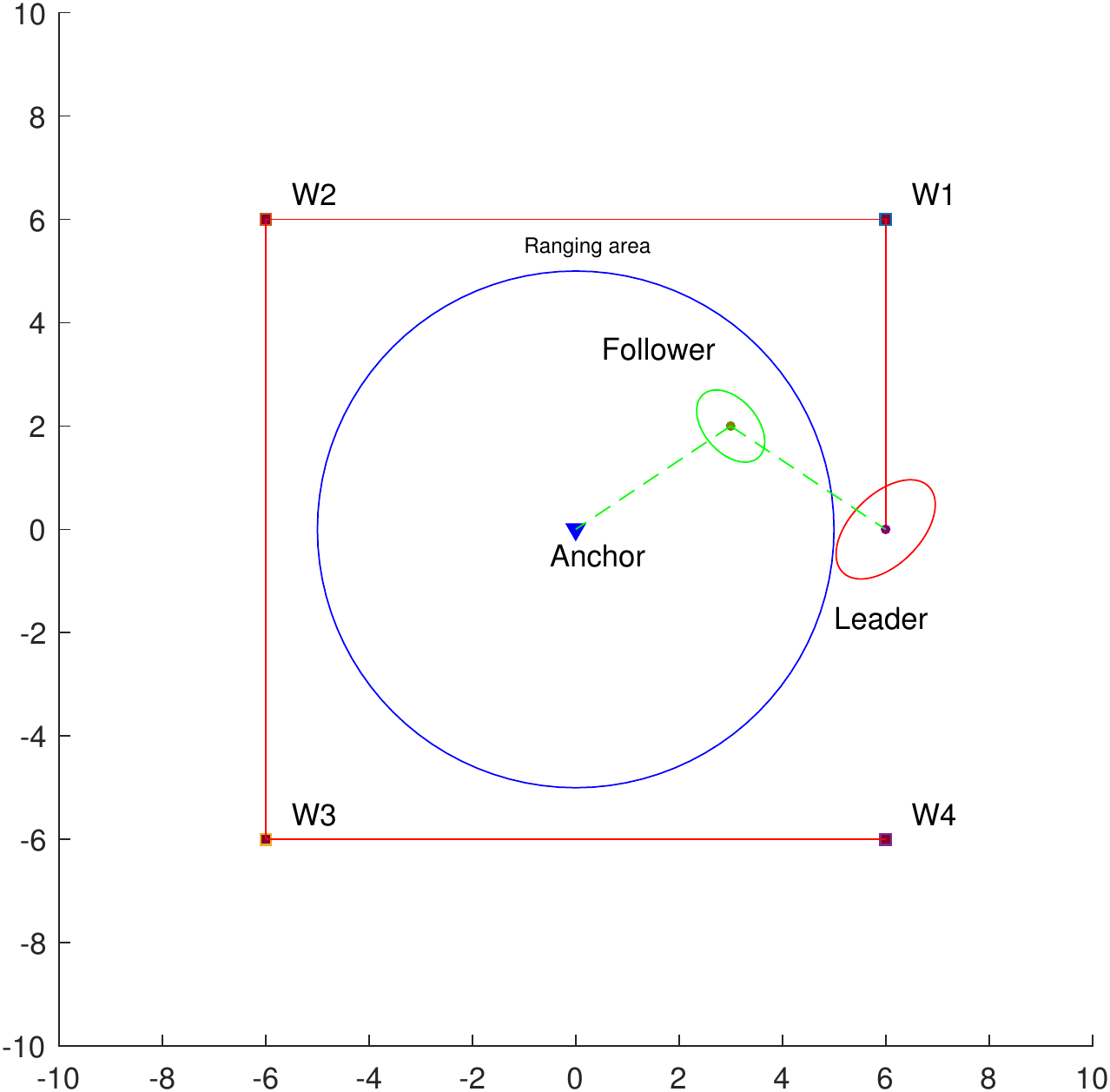}
	\caption{Initial configurations of the CLAP problem. 
	}
	\label{fig. CLAPconfigureation}
\end{figure}

\blue{The motion and measurement models applied in this scenario are the same as the previous simulation experiment}. The planning horizon is set to 3s look-ahead. Denoting $\tilde{\vect{p}}_L^l$ and $\vect{\Sigma}_{\tilde{\vect{p}}_L^l}$ as the predicted mean and covariance of the leader robot while $\tilde{\vect{p}}_F^l$ and $\vect{\Sigma}_{\tilde{\vect{p}}_F^l}$ as those of the follower, then the immediate cost in the objective function is defined as:
\begin{equation}
	c^l \left(  b(\vect{X}^{k+l}, \vect{U}^{k+l-1} \right) = \omega_1 trace(\vect{\Sigma}_{\vect{\tilde{p}}_L^{l}}) + \omega_2 trace(\vect{\Sigma}_{\vect{\tilde{p}}_F^{l}}).
\end{equation}

The initial positions are $\vect{p}^0_L = [6,0]^T$ and $\vect{p}_F^0 = [3,2]^T$. The anchor is located at the origin. Initially there is a connection between the follower and anchor. 
Both robots are initialized with some localization uncertainty.
The leader's high-level trajectory is defined by four sequential waypoints $W_1 = [6,6]^T, W_2 = [-6,6]^T, W_3 = [-6,-6]^T, W_4 = [6,-6]^T$. We calculate the leader's control input by directly heading to its destination waypoint, which switches to next one once $d_{L\rightarrow W_j} = ||\bar{\vect{p}}_L - W_j|| < 0.2m$, where $\vect{\bar{p}}_L$ is the leader's estimated position. The maximum moving distance at each time step for leader and follower are 0.5m and 0.7m, respectively. The gains in the objective function are chosen as $\omega_1=9$ and $\omega_2=1$. \blue{The GBS proposed in \cite{indelman2015planning} is utilized as the active planning framework.}

\begin{figure*}
	\centering
	\begin{subfigure}[b]{0.32\textwidth}
		\begin{subfigure}[b]{1\textwidth}
			\centering
			\includegraphics[width=\textwidth]{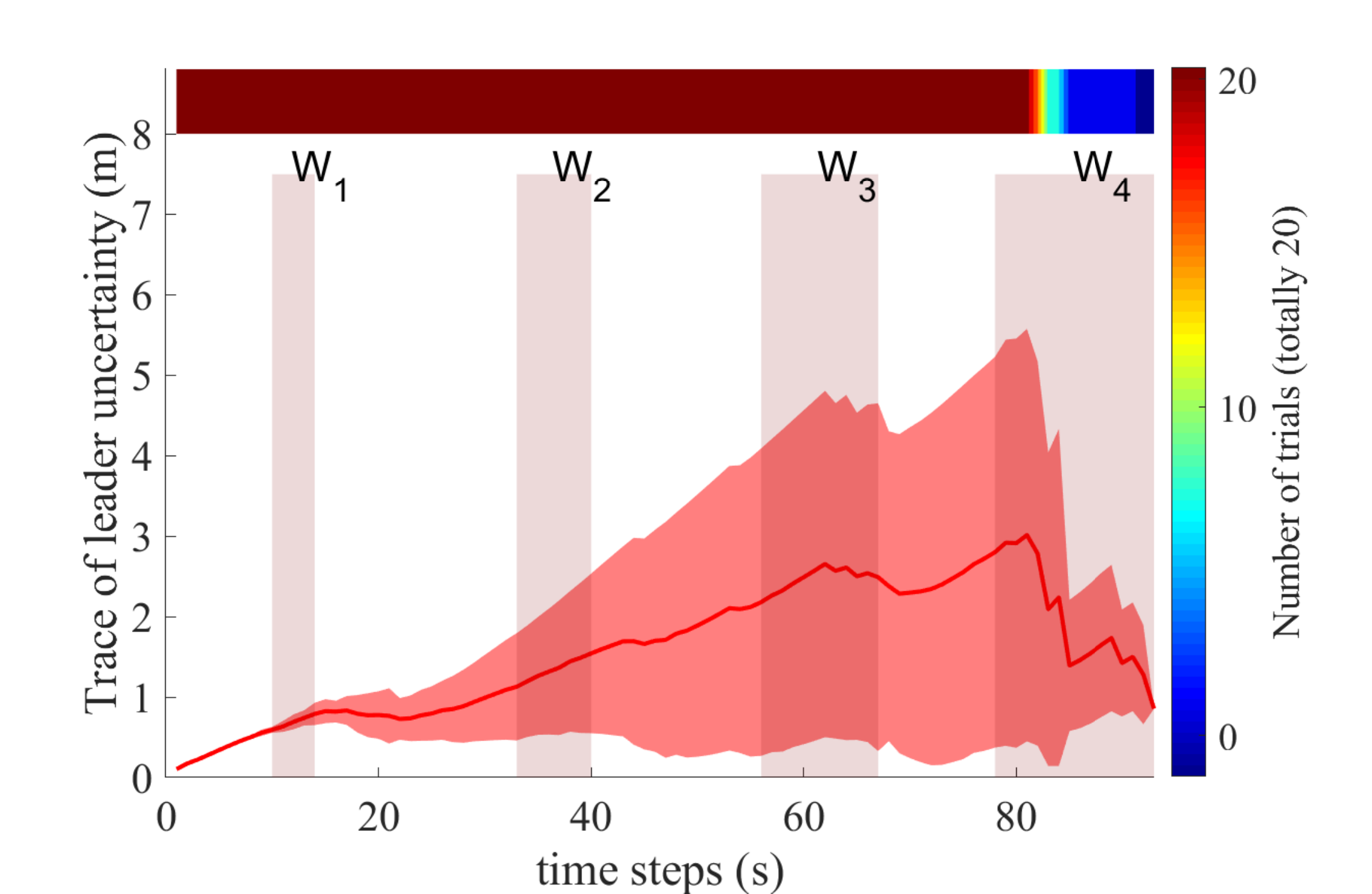}
			\caption{Bernoulli}
			\label{fig. Bernoulli model of CLAP problem distribution over 50 steps}
		\end{subfigure}
		\null \hfill \\
		\begin{subfigure}[b]{1\textwidth}
			\centering
			\includegraphics[width=\textwidth]{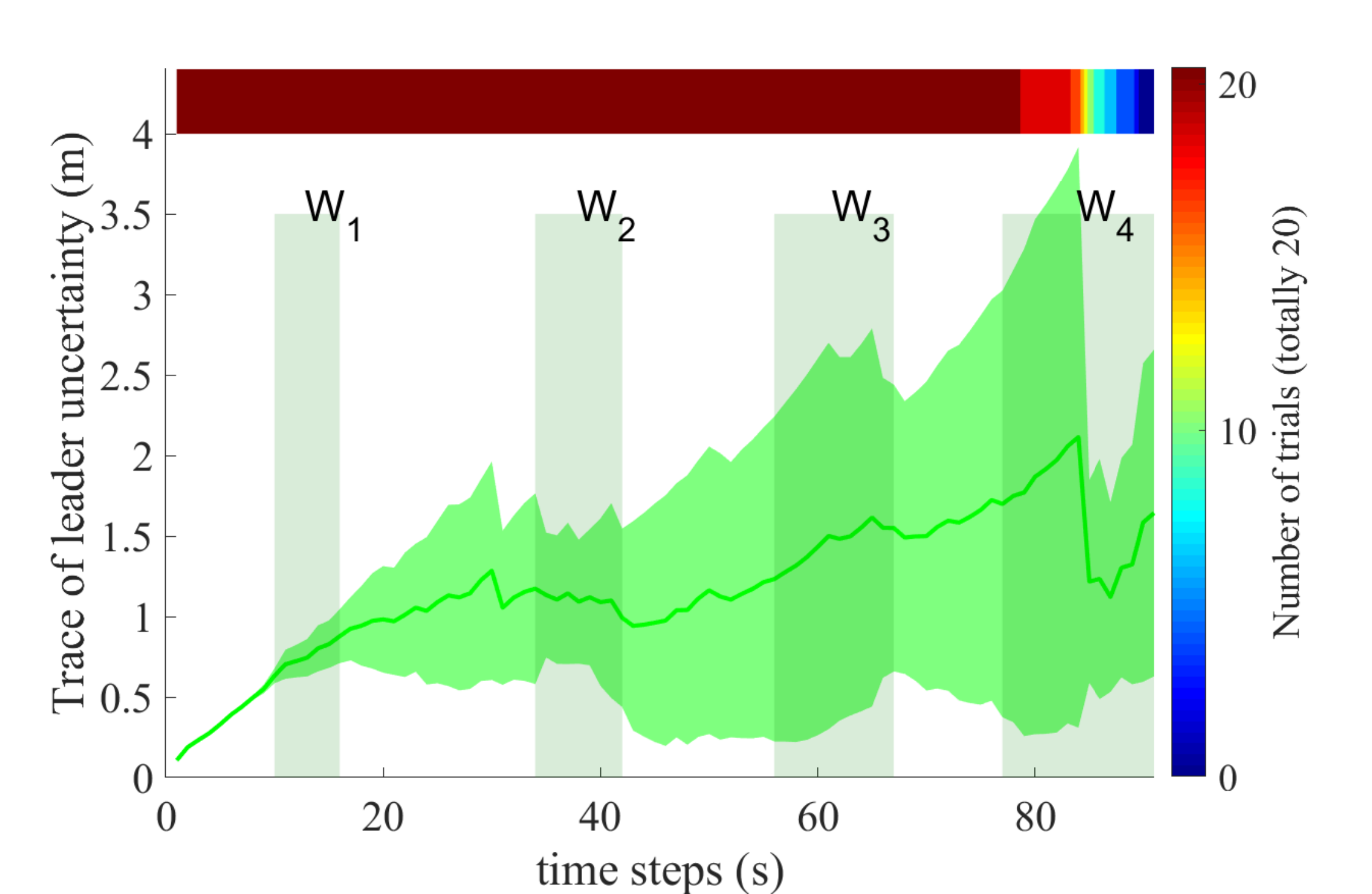}
			\caption{Linear}
			\label{fig. Linear model of CLAP problem distribution over 50 steps}
		\end{subfigure}
	\end{subfigure}
	\hfill
	\begin{subfigure}[b]{0.32\textwidth}
		\begin{subfigure}[b]{1\textwidth}
			\centering
			\includegraphics[width=\textwidth]{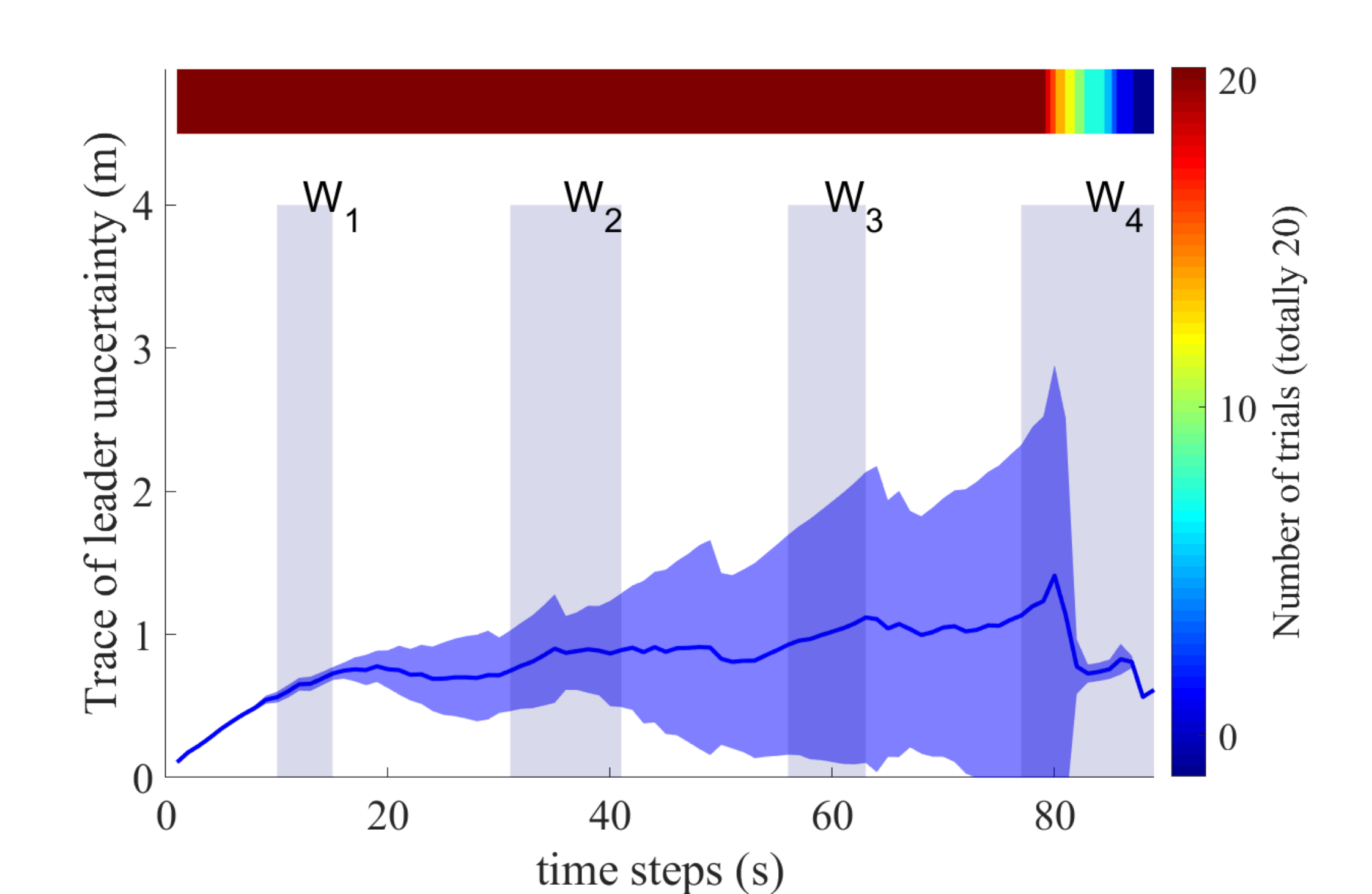}
			\caption{Random with $dim=10^4$}
			\label{fig. Random sample model of CCLAP problem distribution over 50 steps}
		\end{subfigure}
		\null \hfill \\
		\begin{subfigure}[b]{1\textwidth}
			\centering
			\includegraphics[width=\textwidth]{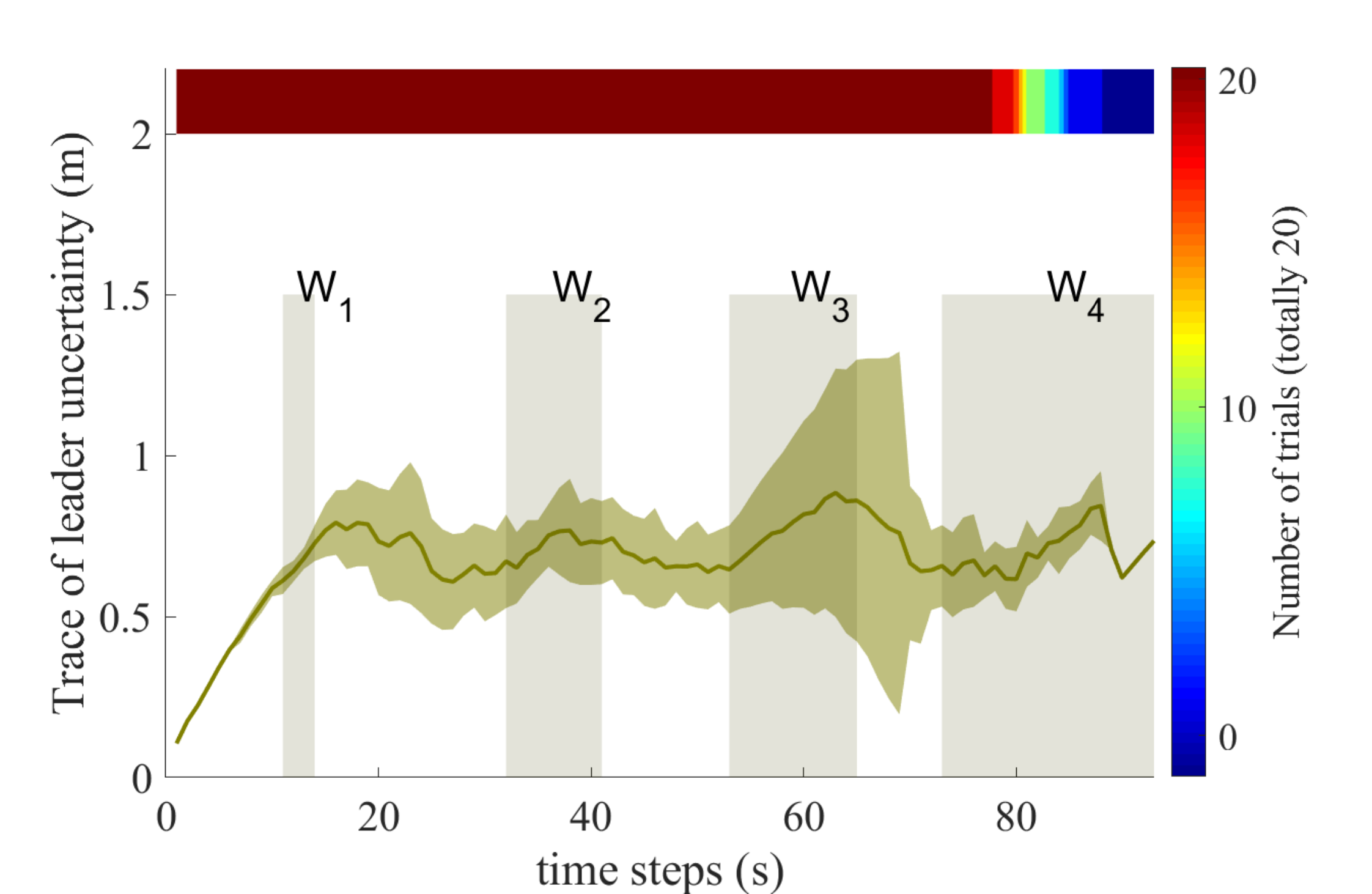}
			\caption{APSE}
			\label{fig. proposed method of CLAP problem distribution over 50 steps}
		\end{subfigure}
	\end{subfigure}
	\hfill
	\begin{subfigure}[b]{0.30\textwidth}
		\begin{subfigure}[b]{1\textwidth}
			\centering
			\includegraphics[width=\textwidth]{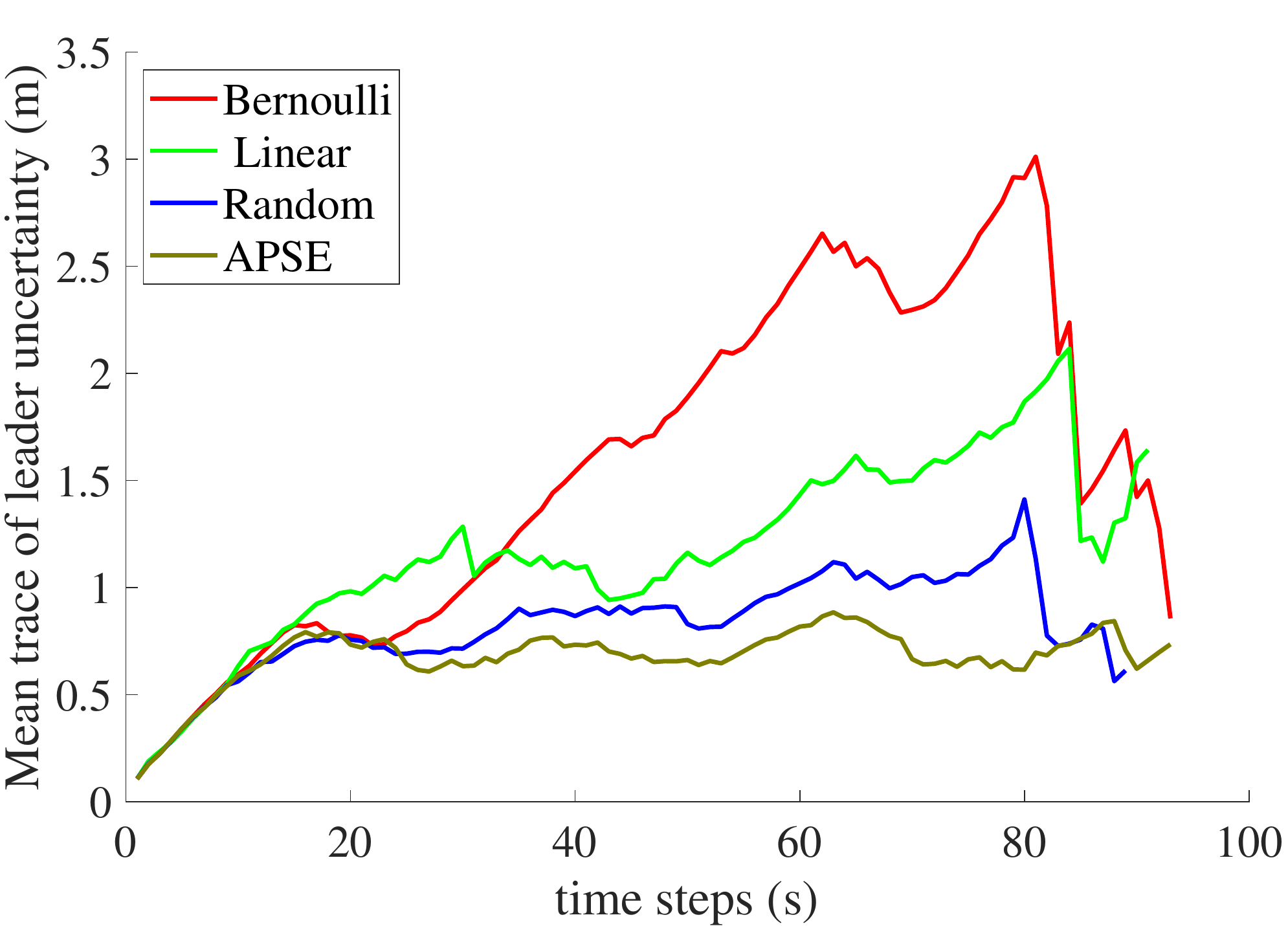}
			\caption{Leader's mean uncertainty}
			\label{fig. CLAPMeanUncertaintyInoneFigure}
		\end{subfigure}
		\null \hfill \\
		\begin{subfigure}[b]{1\textwidth}
			\centering
			\includegraphics[width=\textwidth]{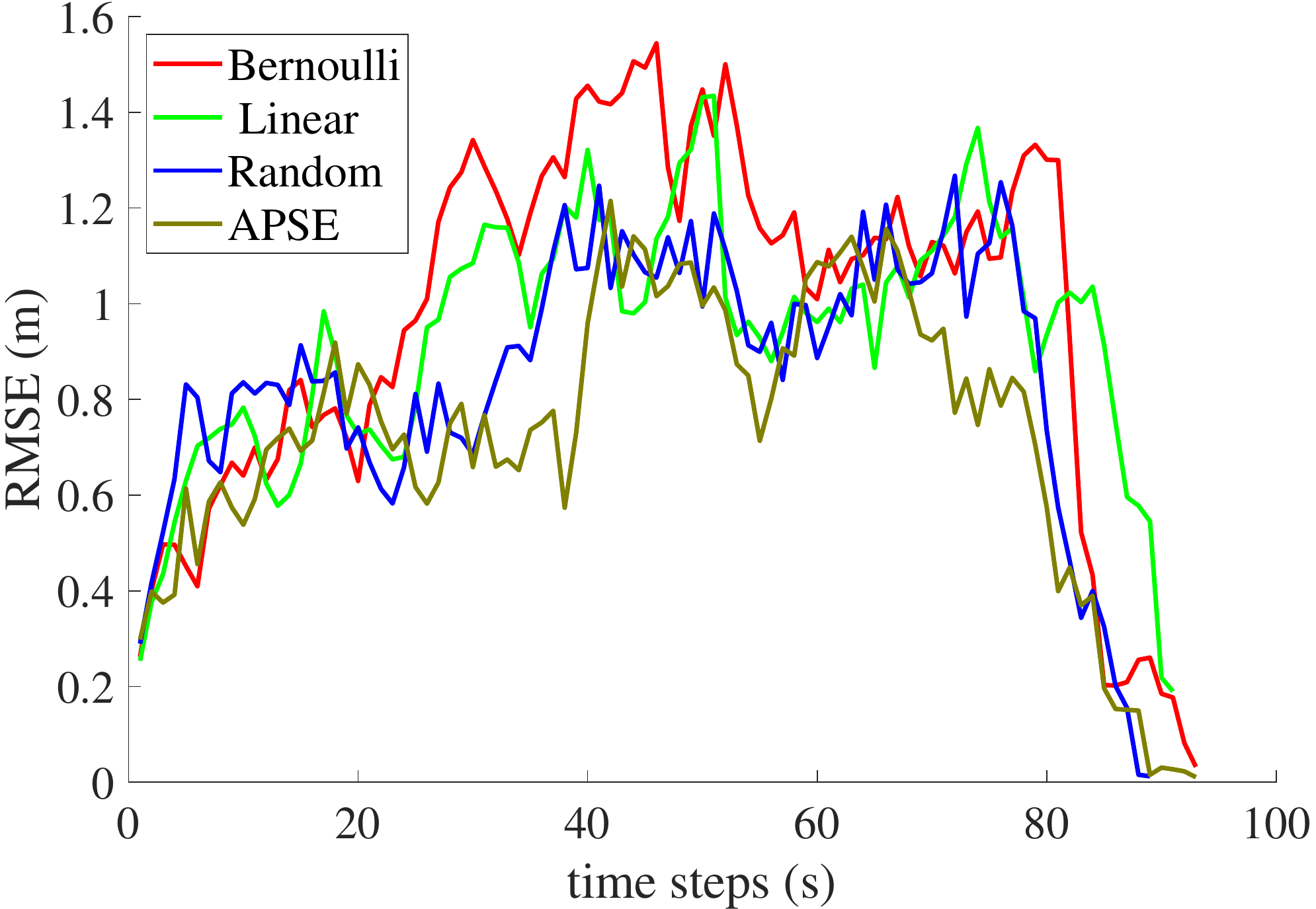}
			\caption{\blue{Leader's mean localization deviation}}
		\label{fig. CLAPMeanDeviation}
		\end{subfigure}
	\end{subfigure} 
	\caption{(a)-(d) Statistical distribution of the leader's localization uncertainty over 20 trials for CLAP with GBS. The vertical shaded bars marked by $W_1$-$W_4$ indicate the distribution of times that leader reaches that waypoint. The colormap shows how many trials exist across the time line. Note that they have different scales on their axes. (e) The leader's mean localization uncertainty for four methods in one figure. (f) The mean RMSE of the leader's localization deviation.}
	\label{fig. CLAPStaisticDistributionUncertainty}
	\vspace{-10pt}
\end{figure*}

Statistical results are presented in Fig.~\ref{fig. CLAPStaisticDistributionUncertainty}\textcolor{red}{a-d} showing the distribution of the leader's localization uncertainty over 20 trials. Notice that the total run time for every trial differs since different control strategies are executed due to varying localization uncertainties. The lighter vertical bars in each distribution show the time ranges when the leader reaches each waypoint. The horizontal colored bar along the top demonstrates the number of active trials at the corresponding time step. 
\begin{table*}[tbp] 
    \vspace{10pt}
	\centering	
	\caption{Statistical characteristics over 20 trials.}
	\begin{tabular}{ccccc}
		\hline
		\multirow{2}{*}{\begin{tabular}[c]{@{}c@{}}Performance metrics\end{tabular}} & \multicolumn{4}{c}{Methods}               \\ \cline{2-5} 
		& Bernoulli & Linear  & Random   & APSE             \\ \hline
		mm-RMSE (m)     & 1.0669    & 1.0381  & 0.9119  & \textbf{0.8079}  \\
		mm-Trace (m)    & 1.5651    & 1.1715  & 0.8208  & \textbf{0.6752}  \\
		\blue{m-TotalCon}        & \blue{41.85 $\pm$ 23.47}      & \blue{44.70 $\pm$ 14.09}     & \blue{66.85 $\pm$ 20.07}    & \blue{\textbf{75.85 $\pm$ 8.85}}    \\
		m-TotalTra (m)  &\blue{47.1090 $\pm$ 1.9839} & \blue{46.9427 $\pm$ 2.4161} & \blue{46.6451 $\pm$ 1.6078}& \textbf{\blue{45.7925 $\pm$ 1.5493}} \\ \hline
	\end{tabular}
	\vspace{-10pt}
	\label{Tab. processedResults}
\end{table*} 

We further collect the processed characteristics for all four methods into one plot to make a clearer comparison as shown in Fig.~\ref{fig. CLAPMeanUncertaintyInoneFigure} and Fig.~\ref{fig. CLAPMeanDeviation}. Here two performance metrics are considered, the evolution of the trace of the leader's uncertainty and the mean RMSE of the leader's localization error (deviation between the estimated and real position). More statistical results are computed and compared in Table~\ref{Tab. processedResults}. The performance metrics considered here are:

1) mm-RMSE: the leader's localization RMSE when taking the mean over all time steps and over all trials. 

2) mm-Trace: trace of the leader's covariance matrix when taking the mean over all time steps and over all trials. 

\blue{3) m-TotalCon: mean over 20 trials of the total number of the relay connections transmitting location information from anchor to leader through follower in each trial. }

4) m-TotalTra: mean distance of the leader's real trajectory for each trial. The distance of the nominal trajectory assuming perfect knowledge of the leader's position is 42m.

By comparing Fig.~\ref{fig. CLAPMeanUncertaintyInoneFigure} and Fig.~\ref{fig. CLAPMeanDeviation}, we can draw two main conclusions, 

1) the methods \{Linear, Random, APSE\} under probabilistic connection perform better than the traditional approach, Bernoulli, under a deterministic connection, 

2) the localization uncertainty becomes more convergent and its mean uncertainty is smaller if a more accurate algorithm, such as APSE, is used to calculate the connection probability. 

From Table~\ref{Tab. processedResults}, an evident conclusion that can be directly derived is that our proposed APSE method surpasses all other approaches \{Bernoulli, Linear, Random\}. The improvement in performance of mm-RMSE for APSE is \{24.28\%, 22.18\%, 11.40\%\} when compared with the other three methods. As for the mm-Trace, its reduction rates are respectively \{56.86\%, 42.36\%, 17.74\%\}. In addition, the comparative increase in the \blue{m-TotalCon} is \{81.24\%, 75.73\%, 13.47\%\}.

Regarding the traveled distance, a more meaningful way is to compare the extra distance, computed by subtracting the nominal shortest trajectory from the executed one. Intuitively, the extra travel distance is introduced due to the existence of uncertainties and hence it can reflect how uncertainties are accumulated across the evolution. As the leader's nominal travel distance is 42m, we say that by applying the underlying four active planning strategies, the extra distances are \{5.1090, 4.9427, 4.6451, 3.7925\}m. Therefore, our proposed APSE achieves a \{25.76\%, 23.27\%, 18.35\%\} reduction on the distance traveled.  


\section{Conclusions and future works}

This note addressed the algorithmic challenges in active planning problems that arise when the future observation/communication connection is unknown at the time of planning. 
The most important contribution is an improved algorithm---APSE---for computing the exact probability of a future connection considering the disk communication and observation model based on available information at the current time and given the control candidates. Through both theoretical analyses and numerical simulations, three conclusions are made as follows:

1) The APSE can achieve a theoretically guaranteed accuracy under an adaptive selection of the summation degree. 

2) The computational complexity of APSE is at the same magnitude as random sampling methods with their number of samples between $10^3$ and $10^4$.
 
3) We have verified the idea that the performance of active planning can indeed be further improved by accurately predicting the distribution of future unknown variables.


However, despite the aforementioned achievements, the APSE algorithm still suffers from round-off error from the computation platform. This error may significantly impact its accuracy if the communication and observation threshold $\rho$ is too large. Though the TRACE algorithm can be used inversely to deal with this problem, the critical requirement on the eigenvalues of the covariance still restricts its performance. A direct way to address this problem is by applying packages with higher precision of quantities such as the Multiple Precision Toolbox for the MATLAB platform \url{https://www.advanpix.com/}. Another roadmap 
may refer to the variants of a standard chi-square distribution where a look-up table is often used to search for the probability and interpolation needed for the values missing in the table. Either Patnaik's approximation \cite{patnaik1949non} or Pearson's approximation  \cite{imhof1961computing} may be an underlying solution.




In addition, more investigations should be conducted to reason about the independence relationships depicted in \eqref{eq. spatiallyIndependenceFutureC}-\eqref{eq. temporallyIndependenceFutureC} regarding future connectivities. Though spatial independence can be established by the separated measurement process, temporal independence should be further looked into, especially in terms of the applicability of the Markov assumption.

Moving on, another important topic is to 
extend our analysis from 3 DOF to 6 DOF. The key challenge in such an extension is how to build a mathematical measurement model and reason about the correlation between the limits of the measurement model caused by future positions and attitudes.

Lastly, more complex environments could be taken into further consideration. Here we only assume an empty world with an ideal disk communication and observation model. However, real deployment of the active planning problem will likely have to deal with different configurations, like physical obstructions in the environment, the blocking of data transmission or radio measurement signals, the separation of communication and observation devices, various types of sensor noise, 
faults, failures and so on.  

\appendices
\section{} \label{subsec proof d_k}
Recalling the definition of the coefficient $d_w$ \eqref{eq. computation of the coefficient d_k}, we know the subtraction of two consecutive terms,  $d_{w+1}$ and $d_w$ is
\begin{equation}
	\begin{aligned}
		&d_{w+1} - d_w  \\
		&= \frac{(2\lambda_j)^{-w}}{2}  \sum_{j=1}^2 \left\lbrace  wb_j^2\left[ 1 - (2\lambda_j)^{-1} \right] + (1-b_j^2)(2\lambda_j)^{-1} - 1\right\rbrace.
	\end{aligned}
\end{equation}

As we have $\left[ 1 - (2\lambda_j)^{-1} \right]>0$, then if 
\begin{equation} \label{eq. condition of k, 1}
	\begin{aligned}
		w & \ge \frac{1}{b_j^2} + \frac{1}{2\lambda_j - 1}
	\end{aligned},  \forall j \in \{ 1,2 \},
\end{equation}
the subtraction has
$$d_{w+1} - d_{w} > 0.$$
Moreover, it's clear that $d_w\le 0, \forall w \in \mathbb{N}^+$ if
\begin{equation} \label{eq. condition of k, 2}
	w \ge \frac{1}{b_j^2},  \forall j \in \{ 1,2 \}. 
\end{equation}

From \eqref{eq. computation of the coefficient d_k}, the limit of $d_w$ is 
$$		\mathop{\lim}_{w\rightarrow \infty} d_w =  \mathop{\lim}_{k\rightarrow \infty} \sum_{j=1}^p \left( \frac{1}{ (2\lambda_j
	)^{w} } - b_j^2\frac{w}{ (2\lambda_j
	)^{w} } \right),\forall j \in {1,2}.$$
Then according to the well known L'Hopital rule, if $2\lambda_j > 1$, the right-hand side of the limit is exactly zero. 

In conclusion, we derive that if the given distribution of relative distance between two nodes is well-defined i.e. $\lambda>\frac{1}{2}$, the output of $|d_w|$ will keep decreasing when the degree $w$ grows large enough to realize the conditions in \eqref{eq. condition of k, 1}, as the condition in \eqref{eq. condition of k, 2} is included in \eqref{eq. condition of k, 1}. Therefore we know that there exists a maximum $d^u > 0$ so that $|d_w| < d^u, \forall w \in \mathbb{R}^{+}$. 

This ends the proof of Corollary \ref{Corollary of d_k}.

\section{}  \label{subsec proof c_k}
Let expand two consecutive terms $\tilde{c}_{w+1}$ and $\tilde{c}_w$, 	
\begin{equation} \label{tilde expansion of c_k}
	k\tilde{c}_w = \left( \tilde{c}_0\tilde{d}_w + \tilde{c}_1\tilde{d}_{w-1} + ... + \tilde{c}_{w-1}\tilde{d}_1\right),
\end{equation}
and
\begin{equation} \label{tilde expansion of c_k+1}
	(w+1) \tilde{c}_{w+1} = \left( \tilde{c}_0\tilde{d}_{w+1} + \tilde{c}_1\tilde{d}_{w} + ... + \tilde{c}_{w}\tilde{d}_1\right).
\end{equation}
Adding $\tilde{c}_w$ to both sides of \eqref{tilde expansion of c_k} derives
\begin{equation} \label{tilde expansion of c_k after adding tilde c_k}
	(w+1)\tilde{c}_w = \left( \tilde{c}_0\tilde{d}_w + \tilde{c}_1\tilde{d}_{w-1} + ... + \tilde{c}_{w-1}\tilde{d}_1 + \tilde{c}_w \right).
\end{equation}
Taking the subtraction of \eqref{tilde expansion of c_k after adding tilde c_k} and \eqref{tilde expansion of c_k+1} has
\begin{equation}
	\begin{aligned}
		&(w+1)\left( \tilde{c}_{w+1} - \tilde{c}_{w} \right) \\
		&= \tilde{c}_0 (\tilde{d}_{w+1} - \tilde{d}_{w} ) + ... + \tilde{c}_{w-1} (\tilde{d}_{2} - \tilde{d}_{1} )  +\tilde{c}_w (\tilde{d}_{1} - 1 ).  
	\end{aligned}		
\end{equation}
Since $\tilde{d}_{w+1} = \tilde{d}_{w} = ... = \tilde{d}_{1} = d^u$, we can further get 
\begin{equation} 
	\begin{aligned}
		(w+1)&\left( \tilde{c}_{w+1} - \tilde{c}_{w} \right) = (d^u - 1)\tilde{c}_w \\
		\Rightarrow \quad & \tilde{c}_{w+1} = \frac{d^u+w}{w+1} \tilde{c}_{w}. 
	\end{aligned}		
\end{equation}
Recalling that $\tilde{c}_0 = c_0 > 0$, then we know $\tilde{c}_w > 0,\forall w \in \mathbb{N}^+$. 

Next, we need to show that $|c_w| \le |\tilde{c}_w|,\forall w \in \mathbb{N}^+$. It has been already shown that $c_0 = \tilde{c}_0 \Rightarrow |c_0| \le |\tilde{c}_0|$. Hence, we have 
\begin{equation}
	|c_1| = |c_0d_1| \le |c_0| |d^u| = |\tilde{c}_1|.
\end{equation}
Similarly, we suppose that $|c_i| \le |\tilde{c}_i|$ when $i=w-1$. Then we can derive, 
\begin{equation} 
	\begin{aligned}
		|c_w| &= |\frac{1}{w} \left( c_0d_w + c_1d_{w-1} + ... + c_{w-1}d_1 \right)| \\
		& \le \frac{1}{w} \left( |\tilde{c}_0||d^u| + |\tilde{c}_1||d^u| + ... + |\tilde{c}_{w-1}||d^u| \right) = |\tilde{c}_w| \\
	\end{aligned}	
\end{equation}
Note that the above equation is held by the fact that two new series $\tilde{c}_w$ and $\tilde{d}_w$ are all positive. 

Similarly, we can also conclude the same result as shown above when $i > w - 1$. Therefore, the conclusion in \eqref{eq. abstract relationship of c_k and tilde_c_k} can be drawn. This ends the proof of Corollary \ref{corollary of tilde_c_k}.

\section{}  \label{subsec proof e_k^D}

A similar series $\tilde{e}_w^D$ can be built using $\tilde{c}_w$ according to the definition of $e_w^D$ in \eqref{eq. new series e_k^D, D = 1} and \eqref{eq. new series e_k^D, D>2}.
Therefore, the relationship between $c_w$ and $\tilde{c}_w$ in \eqref{eq. abstract relationship of c_k and tilde_c_k} can be equally extended to $e_w^D$ and $\tilde{e}_w^D$, whereby we know that
$$|e_w^D| \le |\tilde{e}_w^D|,\forall w \in \mathbb{N}^+.$$	

According to \eqref{eq. recurrsive equation of c_k+1 and c_k given a equal d_k}, the following extension can be made when $w \ge D-1$, 
\begin{equation} \label{eq. tilde_e_k^D and tilde_e_{w+1}^{D}}
	\tilde{e}_{w+1}^D = \frac{d^u+w}{w+1} \frac{\prod_{j=1}^{D}(w+2-j)}{\prod_{j=1}^{D}(w+3-j)}\tilde{e}_{w}^D.		
\end{equation}
So, if $D = 1$, we derive,
\begin{equation} \label{eq. tilde_e_k^D if D=1}
	\begin{aligned}
		\tilde{e}_{w+1}^1 &= \frac{d^u+w}{w+2} \tilde{e}_{w}^1 < \tilde{e}_{w}^1, \quad \forall w \ge 0.\\
	\end{aligned}		
\end{equation}
This means $\tilde{e}_w^1$ gets smaller when $w$ tends to infinity.

Otherwise, when $D$ is general and $D \ge 2$, we have
\begin{equation} \label{eq. tilde_e_k^D if D ge 2}
	\begin{aligned}
		\tilde{e}_{w+1}^D &= \frac{d^u+w}{w+1} \frac{(w+1)(w+2-D)\prod_{j=2}^{D-1}(w+2-j)}{(w+2)(w+1)\prod_{j=3}^{D}(w+3-j)}\tilde{e}_{w}^D  \\
		& \le \frac{ (d^u+w)(w+2-D) }{ (w+2)(w+1) } \frac{\prod_{j=2}^{D-1}(w+2-j)}{\prod_{j=2}^{D-1}(w+2-j)}\tilde{e}_{w}^D \\
		& = \frac{ (d^u+w)(w+2-D) }{ (w+2)(w+1) }\tilde{e}_{w}^D\quad \forall w \ge D-1.
	\end{aligned}		
\end{equation}
A simple calculation derives, 
$$
\begin{aligned}
	&(d^u+w)(w+2-D) - (w+2)(w+1) \\
	&= (d^u-D-1)w +d^u(2-D)-2.
\end{aligned}
$$
As we know that $0 \le d^u \le D$ and $D \ge 2$, hence $(d^u-D-1) < 0$ and $d^u(2-D)-2 < 0$. As a result, $\tilde{e}_{w}^D$ is also decreasing when $w\ge D-1$, i.e. $	\tilde{e}_{w+1}^D < \tilde{e}_{w}^D, \quad \forall w \ge D-1.$
Therefore, the limit of $\tilde{e}^D_k$ is bounded,
$\mathop{\lim}_{w \rightarrow \infty} \tilde{e}_{w}^D = \mathit{Const.}$

Since $\tilde{e}_w^D$ is the envelope of $e_w^D$, then the limit of $|e_w^D|$ is also bounded, which gives the result in \eqref{eq. limitation of e_k^D}. The maximum of $\tilde{e}_w^1$ is obviously the first element $e_0^1 = \tilde{c}_0$. As for $D \ge 2$, we recall the relationship of $\tilde{c}_w$ and $\tilde{c}_{w+1}$ in \eqref{eq. recurrsive equation of c_k+1 and c_k given a equal d_k}. We know that $\tilde{e}_w^D$ is increasing when $w \le D-2$ and the last element $\tilde{e}_{D-2}^D = \tilde{c}_{D-2}$ is the largest value thus far. After that, $\tilde{e}_{w}^D$ will decrease over the increment of $w$, which gives its peak value $\tilde{e}_{D-1}^D = \frac{\tilde{c}_{D-1}}{D!}\le\tilde{c}_{D-1}$ when $w\ge D-1$. Since we have $\tilde{c}_{D-2} < \tilde{c}_{D-1}$, the maximum value of $\tilde{e}_w^D$ is therefore $\tilde{c}_{D-1}$. Consequently, it is also the upper boundary of $|e_w^D|$. This ends the proof of Corollary \ref{corollary of e_k^D}.


%


%
%

\ifCLASSOPTIONcaptionsoff
  \newpage
\fi



%
%
%
\bibliographystyle{IEEEtran}
\bibliography{IEEEabrv,mybibfile}

\begin{thebibliography}{10}
\providecommand{\url}[1]{#1}
\csname url@rmstyle\endcsname
\providecommand{\newblock}{\relax}
\providecommand{\bibinfo}[2]{#2}
\providecommand\BIBentrySTDinterwordspacing{\spaceskip=0pt\relax}
\providecommand\BIBentryALTinterwordstretchfactor{4}
\providecommand\BIBentryALTinterwordspacing{\spaceskip=\fontdimen2\font plus
\BIBentryALTinterwordstretchfactor\fontdimen3\font minus
  \fontdimen4\font\relax}
\providecommand\BIBforeignlanguage[2]{{%
\expandafter\ifx\csname l@#1\endcsname\relax
\typeout{** WARNING: IEEEtran.bst: No hyphenation pattern has been}%
\typeout{** loaded for the language `#1'. Using the pattern for}%
\typeout{** the default language instead.}%
\else
\language=\csname l@#1\endcsname
\fi
#2}}

\bibitem{cadena2016past}
C.~Cadena, L.~Carlone, H.~Carrillo, Y.~Latif, D.~Scaramuzza, J.~Neira, I.~Reid,
  and J.~J. Leonard, ``Past, present, and future of simultaneous localization
  and mapping: Toward the robust-perception age,'' \emph{IEEE Transactions on
  Robotics}, vol.~32, no.~6, pp. 1309--1332, 2016.

\bibitem{hidaka2005optimal}
Y.~S. Hidaka, A.~I. Mourikis, and S.~I. Roumeliotis, ``Optimal formations for
  cooperative localization of mobile robots,'' in \emph{Proceedings of the 2005
  IEEE International Conference on Robotics and Automation}.\hskip 1em plus
  0.5em minus 0.4em\relax IEEE, 2005, pp. 4126--4131.

\bibitem{schlotfeldt2018anytime}
B.~Schlotfeldt, D.~Thakur, N.~Atanasov, V.~Kumar, and G.~J. Pappas, ``Anytime
  planning for decentralized multirobot active information gathering,''
  \emph{IEEE Robotics and Automation Letters}, vol.~3, no.~2, pp. 1025--1032,
  2018.

\bibitem{patwari2005locating}
N.~Patwari, J.~N. Ash, S.~Kyperountas, A.~O. Hero, R.~L. Moses, and N.~S.
  Correal, ``Locating the nodes: Cooperative localization in wireless sensor
  networks,'' \emph{IEEE Signal Processing Magazine}, vol.~22, no.~4, pp.
  54--69, 2005.

\bibitem{platt2010belief}
R.~Platt, R.~Tedrake, L.~Kaelbling, and T.~Lozano-Perez, ``Belief space
  planning assuming maximum likelihood observations,'' in \emph{Proceedings of
  Robotics: Science and Systems}, Zaragoza, Spain, June 2010.

\bibitem{indelman2015planning}
V.~Indelman, L.~Carlone, and F.~Dellaert, ``Planning in the continuous domain:
  A generalized belief space approach for autonomous navigation in unknown
  environments,'' \emph{The International Journal of Robotics Research},
  vol.~34, no.~7, pp. 849--882, 2015.

\bibitem{pathak2018unified}
S.~Pathak, A.~Thomas, and V.~Indelman, ``A unified framework for data
  association aware robust belief space planning and perception,'' \emph{The
  International Journal of Robotics Research}, vol.~37, no. 2-3, pp. 287--315,
  2018.

\bibitem{farhi2019ix}
E.~I. Farhi and V.~Indelman, ``ix-bsp: Belief space planning through
  incremental expectation,'' in \emph{2019 International Conference on Robotics
  and Automation (ICRA)}.\hskip 1em plus 0.5em minus 0.4em\relax IEEE, 2019,
  pp. 7180--7186.

\bibitem{indelman2018cooperative}
V.~Indelman, ``Cooperative multi-robot belief space planning for autonomous
  navigation in unknown environments,'' \emph{Autonomous Robots}, vol.~42,
  no.~2, pp. 353--373, 2018.

\bibitem{zhang2020connectivity}
L.~Zhang, Z.~Zhang, R.~Siegwart, and J.~J. Chung, ``A connectivity-prediction
  algorithm and its application in active cooperative localization for
  multi-robot systems,'' in \emph{2020 IEEE International Conference on
  Robotics and Automation (ICRA)}.\hskip 1em plus 0.5em minus 0.4em\relax IEEE,
  2020, pp. 9824--9830.

\bibitem{provost1992quadratic}
A.~Mathai and S.~B. Provost, \emph{Quadratic Forms in Random Variables: Theory
  and Applications}.\hskip 1em plus 0.5em minus 0.4em\relax Marcel Dekker,
  Inc., 1992.

\bibitem{puterman2014markov}
M.~L. Puterman, \emph{Markov Decision Processes: Discrete Stochastic Dynamic
  Programming}.\hskip 1em plus 0.5em minus 0.4em\relax John Wiley \& Sons,
  2014.

\bibitem{denardo2012dynamic}
E.~V. Denardo, \emph{Dynamic Programming: Models and Applications}.\hskip 1em
  plus 0.5em minus 0.4em\relax Courier Corporation, 2012.

\bibitem{shani2013survey}
G.~Shani, J.~Pineau, and R.~Kaplow, ``A survey of point-based {POMDP}
  solvers,'' \emph{Autonomous Agents and Multi-Agent Systems}, vol.~27, no.~1,
  pp. 1--51, 2013.

\bibitem{Hsu2008PointTargetTracking}
D.~{Hsu}, {Wee Sun Lee}, and N.~{Rong}, ``A point-based {POMDP} planner for
  target tracking,'' in \emph{2008 IEEE International Conference on Robotics
  and Automation}, 2008, pp. 2644--2650.

\bibitem{wang2016improved}
J.~Wang, X.~Li, and M.~Q.-H. Meng, ``An improved {RRT} algorithm incorporating
  obstacle boundary information,'' in \emph{2016 IEEE International Conference
  on Robotics and Biomimetics}.\hskip 1em plus 0.5em minus 0.4em\relax IEEE,
  2016, pp. 625--630.

\bibitem{prentice2009belief}
S.~Prentice and N.~Roy, ``The belief roadmap: Efficient planning in belief
  space by factoring the covariance,'' \emph{The International Journal of
  Robotics Research}, vol.~28, no. 11-12, pp. 1448--1465, 2009.

\bibitem{agha2014firm}
A.-A. Agha-Mohammadi, S.~Chakravorty, and N.~M. Amato, ``Firm: Sampling-based
  feedback motion-planning under motion uncertainty and imperfect
  measurements,'' \emph{The International Journal of Robotics Research},
  vol.~33, no.~2, pp. 268--304, 2014.

\bibitem{kopitkov2017no}
D.~Kopitkov and V.~Indelman, ``No belief propagation required: Belief space
  planning in high-dimensional state spaces via factor graphs, the matrix
  determinant lemma, and re-use of calculation,'' \emph{The International
  Journal of Robotics Research}, vol.~36, no.~10, pp. 1088--1130, 2017.

\bibitem{kopitkov2019general}
------, ``General-purpose incremental covariance update and efficient belief
  space planning via a factor-graph propagation action tree,'' \emph{The
  International Journal of Robotics Research}, vol.~38, no.~14, pp. 1644--1673,
  2019.

\bibitem{hsieh2008maintaining}
M.~A. Hsieh, A.~Cowley, V.~Kumar, and C.~J. Taylor, ``Maintaining network
  connectivity and performance in robot teams,'' \emph{Journal of Field
  Robotics}, vol.~25, no. 1-2, pp. 111--131, 2008.

\bibitem{fiedler1973algebraic}
M.~Fiedler, ``Algebraic connectivity of graphs,'' \emph{Czechoslovak
  Mathematical Journal}, vol.~23, no.~2, pp. 298--305, 1973.

\bibitem{kim2005maximizing}
Y.~Kim and M.~Mesbahi, ``On maximizing the second smallest eigenvalue of a
  state-dependent graph {Laplacian},'' in \emph{Proceedings of the 2005,
  American Control Conference, 2005.}\hskip 1em plus 0.5em minus 0.4em\relax
  IEEE, 2005, pp. 99--103.

\bibitem{yang2010decentralized}
P.~Yang, R.~A. Freeman, G.~J. Gordon, K.~M. Lynch, S.~S. Srinivasa, and
  R.~Sukthankar, ``Decentralized estimation and control of graph connectivity
  for mobile sensor networks,'' \emph{Automatica}, vol.~46, no.~2, pp.
  390--396, 2010.

\bibitem{franceschelli2013decentralized}
M.~Franceschelli, A.~Gasparri, A.~Giua, and C.~Seatzu, ``Decentralized
  estimation of {Laplacian} eigenvalues in multi-agent systems,''
  \emph{Automatica}, vol.~49, no.~4, pp. 1031--1036, 2013.

\bibitem{zavlanos2007potential}
M.~M. Zavlanos and G.~J. Pappas, ``Potential fields for maintaining
  connectivity of mobile networks,'' \emph{IEEE Transactions on Robotics},
  vol.~23, no.~4, pp. 812--816, 2007.

\bibitem{zavlanos2011graph}
M.~M. Zavlanos, M.~B. Egerstedt, and G.~J. Pappas, ``Graph-theoretic
  connectivity control of mobile robot networks,'' \emph{Proceedings of the
  IEEE}, vol.~99, no.~9, pp. 1525--1540, 2011.

\bibitem{kantaros2016distributed}
Y.~Kantaros and M.~M. Zavlanos, ``Distributed intermittent connectivity control
  of mobile robot networks,'' \emph{IEEE Transactions on Automatic Control},
  vol.~62, no.~7, pp. 3109--3121, 2016.

\bibitem{banfi2018strategies}
J.~Banfi, A.~Quattrini~Li, I.~Rekleitis, F.~Amigoni, and N.~Basilico,
  ``Strategies for coordinated multirobot exploration with recurrent
  connectivity constraints,'' \emph{Autonomous Robots}, vol.~42, no.~4, pp.
  875--894, 2018.

\bibitem{roumeliotis2003analysis}
S.~I. Roumeliotis and I.~M. Rekleitis, ``Analysis of multirobot localization
  uncertainty propagation,'' in \emph{Proceedings 2003 IEEE/RSJ International
  Conference on Intelligent Robots and Systems}, vol.~2.\hskip 1em plus 0.5em
  minus 0.4em\relax IEEE, 2003, pp. 1763--1770.

\bibitem{mourikis2006performance}
A.~I. Mourikis and S.~I. Roumeliotis, ``Performance analysis of multirobot
  cooperative localization,'' \emph{IEEE Transactions on Robotics}, vol.~22,
  no.~4, pp. 666--681, 2006.

\bibitem{kurazume2000experimental}
R.~Kurazume and S.~Hirose, ``An experimental study of a cooperative positioning
  system,'' \emph{Autonomous Robots}, vol.~8, no.~1, pp. 43--52, 2000.

\bibitem{trawny2004optimized}
N.~Trawny and T.~Barfoot, ``Optimized motion strategies for cooperative
  localization of mobile robots,'' in \emph{IEEE International Conference on
  Robotics and Automation}.\hskip 1em plus 0.5em minus 0.4em\relax IEEE, 2004,
  pp. 1027--1032.

\bibitem{shen2010fundamental1}
Y.~Shen and M.~Z. Win, ``Fundamental limits of wideband localization—{Part
  I}: A general framework,'' \emph{IEEE Transactions on Information Theory},
  vol.~56, no.~10, pp. 4956--4980, 2010.

\bibitem{shen2010fundamental2}
Y.~Shen, H.~Wymeersch, and M.~Z. Win, ``Fundamental limits of wideband
  localization—{Part II}: Cooperative networks,'' \emph{IEEE Transactions on
  Information Theory}, vol.~56, no.~10, pp. 4981--5000, 2010.

\bibitem{cottle1974manifestations}
R.~W. Cottle, ``Manifestations of the {Schur} complement,'' \emph{Linear
  Algebra and its Applications}, vol.~8, no.~3, pp. 189--211, 1974.

\bibitem{zhou2011multirobot}
K.~Zhou and S.~I. Roumeliotis, ``Multirobot active target tracking with
  combinations of relative observations,'' \emph{IEEE Transactions on
  Robotics}, vol.~27, no.~4, pp. 678--695, 2011.

\bibitem{mourikis2006optimal}
A.~I. Mourikis and S.~I. Roumeliotis, ``Optimal sensor scheduling for
  resource-constrained localization of mobile robot formations,'' \emph{IEEE
  Transactions on Robotics}, vol.~22, no.~5, pp. 917--931, 2006.

\bibitem{win2018network}
M.~Z. Win, W.~Dai, Y.~Shen, G.~Chrisikos, and H.~V. Poor, ``Network operation
  strategies for efficient localization and navigation,'' \emph{Proceedings of
  the IEEE}, vol. 106, no.~7, pp. 1224--1254, 2018.

\bibitem{levine2013information}
D.~Levine, B.~Luders, and J.~P. How, ``Information-theoretic motion planning
  for constrained sensor networks,'' \emph{Journal of Aerospace Information
  Systems}, vol.~10, no.~10, pp. 476--496, 2013.

\bibitem{zhong2011distributed}
M.~Zhong and C.~G. Cassandras, ``Distributed coverage control and data
  collection with mobile sensor networks,'' \emph{IEEE Transactions on
  Automatic Control}, vol.~56, no.~10, pp. 2445--2455, 2011.

\bibitem{regev2016multi}
T.~Regev and V.~Indelman, ``Multi-robot decentralized belief space planning in
  unknown environments via efficient re-evaluation of impacted paths,'' in
  \emph{2016 IEEE/RSJ International Conference on Intelligent Robots and
  Systems}.\hskip 1em plus 0.5em minus 0.4em\relax IEEE, 2016, pp. 5591--5598.

\bibitem{regev2018decentralized}
------, ``Decentralized multi-robot belief space planning in unknown
  environments via identification and efficient re-evaluation of impacted
  paths,'' \emph{Autonomous Robots}, vol.~42, no.~4, pp. 691--713, 2018.

\bibitem{thrun2002probabilistic}
S.~Thrun, ``Probabilistic robotics,'' \emph{Communications of the ACM},
  vol.~45, no.~3, pp. 52--57, 2002.

\bibitem{kotz1967series1}
S.~Kotz, N.~L. Johnson, and D.~Boyd, ``Series representations of distributions
  of quadratic forms in normal variables. i. central case,'' \emph{The Annals
  of Mathematical Statistics}, vol.~38, no.~3, pp. 823--837, 1967.

\bibitem{kotz1967series2}
------, ``Series representations of distributions of quadratic forms in normal
  variables ii. non-central case,'' \emph{The Annals of Mathematical
  Statistics}, vol.~38, no.~3, pp. 838--848, 1967.

\bibitem{patnaik1949non}
P.~Patnaik, ``The non-central $\chi$ 2-and {$F$}-distribution and their
  applications,'' \emph{Biometrika}, vol.~36, no. 1/2, pp. 202--232, 1949.

\bibitem{imhof1961computing}
J.-P. Imhof, ``Computing the distribution of quadratic forms in normal
  variables,'' \emph{Biometrika}, vol.~48, no. 3/4, pp. 419--426, 1961.

\end{thebibliography}

%




\begin{IEEEbiography}[{\includegraphics[width=1in,height=1.25in,clip,keepaspectratio]{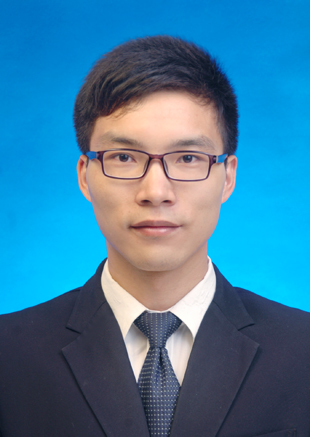}}]{Liang Zhang} Liang Zhang is currently a lecturer in the School of Engineering and Automation, Anhui University, China. His research interests are mainly around the multi-robot or multi-agent system (MAS/MRS), including cooperative control, planning under uncertainty and coverage control. He received his B.Sc. degree in Automation from Shandong University, China in 2015, Ph.D. degree in the School of Astronautics, Harbin Institute of Technology, China in 2021. He was also a visiting Ph.D. student in the Autonomous System Lab, ETH Z{\"u}rich, Switzerland from 2018 to 2019 under the foundation from China Scholarship Council. 
\end{IEEEbiography}

\begin{IEEEbiography}[{\includegraphics[width=1in,height=1.25in,clip,keepaspectratio]{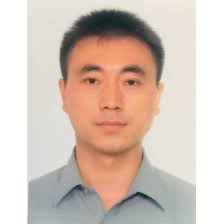}}]{Zexu Zhang}
Zexu Zhang is a full professor at the Harbin Institute of Technology, China, where he is also the director of the Institute of Aircraft Dynamics and Control. His research interests include aircraft autonomous navigation and control, intelligent cooperative perception and autonomous decision-making in drone swarm,  data visualization. He was elected director of the Committee of Space Intelligence of the Chinese Society of Space Research(2021-2025).
\end{IEEEbiography}

\begin{IEEEbiography}[{\includegraphics[width=1in,height=1.25in,clip,keepaspectratio]{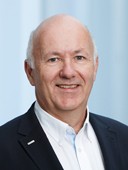}}]{Roland Siegwart}
Roland Siegwart (Fellow, IEEE) is a Professor for autonomous mobile robots with ETH Zurich, Founding Co-Director of the Technology Transfer Center, Wyss Zurich and Board Member of multiple high-tech companies. From 1996 to 2006, he was a Professor with EPFL Lausanne, held visiting positions with Stanford University and NASA Ames and was Vice President of ETH Zurich from 2010 to 2014. His research interest include the design, control, and navigation of flying, and wheeled and walking robots operating in complex and highly dynamical environments.

Prof. Siegwart received the IEEE RAS Pioneer Award and IEEE RAS Inaba Technical Award. He is among the most cited scientist in robots world-wide, Co-Founder of more than half a dozen spin-off companies, and a strong promoter of innovation and entrepreneurship in Switzerland.
\end{IEEEbiography}

\begin{IEEEbiography}[{\includegraphics[width=1in,height=1.25in,clip,keepaspectratio]{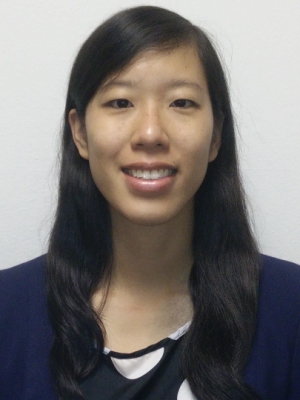}}]{Jen Jen Chung}
Jen Jen Chung (Member, IEEE) is a Senior Researcher in the Autonomous Systems Lab (ASL) at ETH Zürich. Her current research interests include perception and learning for mobile manipulation, algorithms for robot navigation through crowds, informative path planning and adaptive sampling. Prior to working at ASL, Jen Jen was a postdoctoral scholar at Oregon State University researching multiagent learning methods and she completed her Ph.D. on information-based exploration-exploitation strategies for autonomous soaring platforms at the Australian Centre for Field Robotics in the University of Sydney. She received her Ph.D. (2014) and B.E. (2010) from the University of Sydney.

\end{IEEEbiography}




\end{document}